\documentclass{article} % For LaTeX2e
\usepackage{iclr2024_conference,times}

\usepackage{hyperref}
\usepackage{url}

\usepackage{amsmath}
\usepackage{amssymb}
\usepackage{amsfonts}       
\usepackage{algorithm}
\usepackage{algpseudocode}
\usepackage{graphicx}
\usepackage{nicefrac}      
\usepackage{amsthm}
\usepackage{multirow} 
\usepackage{multicol} 
\usepackage{subcaption}
\usepackage{wrapfig}
\usepackage{cancel}
\usepackage{soul}
\usepackage[disable]{todonotes}

\usepackage{localmacros}

\graphicspath{{figures/}}

\theoremstyle{definition}
\newtheorem{proposition}{Proposition}
\newtheorem{lemma}{Lemma}
\newtheorem{definition}{Definition}
\newtheorem{example}{Example}

\numberwithin{equation}{section}

\title{
    Interpretable Meta-Learning \\ of Physical Systems}

\author{Matthieu Blanke
   \\
   Inria Paris, DI ENS, PSL Research University
   \\
\texttt{matthieu.blanke@inria.fr}
\And
Marc Lelarge
\\
Inria Paris, DI ENS, PSL Research University
\\
\texttt{marc.lelarge@inria.fr}
}

\begin{document}

\maketitle

\begin{abstract}
    Machine learning methods can be a valuable aid in the scientific process, but they need to face challenging settings where data come from inhomogeneous experimental conditions. Recently, meta-learning approaches have made significant progress in multi-task learning, but they rely on black-box neural networks, resulting in high computational costs and limited interpretability. Leveraging the structure of the learning problem, we argue that multi-environment generalization can be achieved using a simpler learning model, with an affine structure with respect to the learning task. Crucially, we prove that this architecture can identify the physical parameters of the system, enabling interpretable learning. We demonstrate the competitive generalization performance and the low computational cost of our method by comparing it to state-of-the-art algorithms on physical systems, ranging from toy models to complex, non-analytical systems. The interpretability of our method is illustrated with original applications to physical-parameter-induced adaptation and to adaptive control.
    % and system identification.
    % --------------------------------------------------------------------------------------------------
    % --------------------------------------------------------------------------------------------------
    % --------------------------------------------------------------------------------------------------
    % --------------------------------------------------------------------------------------------------
    % --------------------------------------------------------------------------------------------------
    % --------------------------------------------------------------------------------------------------
    % --------------------------------------------------------------------------------------------------
    % --------------------------------------------------------------------------------------------------
\end{abstract}

% \begin{figure}[H]
%     \centering
%     % \includegraphics[width=\linewidth]{summary.pdf}
%     \includegraphics[height=7cm]{summary.pdf}
%     \label{figure:summary}
%     \caption{Summary of the multi-task learning problem.}
% \end{figure}
% ============================================================
\section{Introduction}

Learning physical systems is an essential application of artificial intelligence that can unlock significant technological and societal progress. Physical systems are inherently complex, making them difficult to learn~\cite{karniadakis2021physics}. A particularly challenging and common scenario is multi-environment learning, where observations of a physical system are collected under inhomogeneous experimental conditions~\cite{caruana1997multitask}. In such cases, the scarcity of training data necessitates the development of robust learning algorithms that can efficiently handle environmental changes and make use of all available data.

This multi-environment learning problem falls within the framework of multi-task learning, which has been widely studied in the field of statistics since the 1990s~\citep{caruana1997multitask}. The aim is to exploit task diversity to learn a shared representation of the data and thus improve generalization. With the rise of deep learning, several meta-learning approaches have attempted in recent years to incorporate multi-task generalization into gradient-based training of deep neural networks.
In the seminal paper by~\cite{finn2017model}, and several variants that followed~\citep{zintgraf2019fast,raghu2019rapid}, this is done by integrating an inner gradient loop in the training process.
% , leading however to second-order optimization.
{Alternatively,~\cite{bertinetto2019meta} proposed adapting the weights using a closed-form solver.}
        % This is typically done by integrating an adaptation loop in the training algorithm, that involves computing and differentiating through an operation on the weights, such as an inner gradient step~\citep{finn2017model, zintgraf2019fast,raghu2019rapid} or a closed-form solver~\citep{bertinetto2019meta}.
As far as physical systems are concerned, the majority of the proposed methods have focused on specific architectures oriented towards trajectory prediction~\citep{wang2022generalizing, pmlr-v162-kirchmeyer22a}.
% , and combine multi-task deep learning with differentiable solvers like Neural Differential Equations~\cite{chen2018neural}. 

% Interpretability~\cite{lipton2018mythos, grojean2022lessons}
When learning a physical system from data, a critical yet often overlooked challenge is model interpretability~\citep{lipton2018mythos, grojean2022lessons}. Interpreting the learned parameters in terms of the system's physical quantities is crucial to making the model more explainable, allowing for scientific discovery and downstream model-based applications such as control.
In a multi-task learning setting, the diversity in the learning environments should enable the identification of the physical parameters that vary across the tasks.
%  the realm of machine learning for understanding physical systems, model interpretability emerges as a critical challenge. Interpreting the fundamental principles governing physical phenomena is essential for both scientific discovery and practical applications. However, the black-box nature of neural networks hinders the ability to gain insights, validate against known physical laws, and trust model predictions. This underscores the  need for interpretable machine learning methods to bridge the gap between the power of neural networks and the transparency required when dealing with the complexities of the physical world.

The above approaches benefit from the expressiveness of deep learning, but are costly in terms of computational time, both for learning and for inference.
Furthermore, the complexity and the black-box nature of neural networks hinder the interpretability of the learned parameters, even when the physical system is linearly parametrized. {Recently,~\cite{wang2021bridging} showed theoretically that the learning capabilities of gradient-based meta-learning algorithms could be matched by the simpler architecture of multi-task representation learning with hard parameter sharing, where the heads of a neural network are trained to adapt to multiple tasks~\citep{caruana1997multitask, ruder2017overview}. They also demonstrated empirically that this architecture is competitive against state-of-the-art gradient-based meta-learning algorithms for few-shot image classification.} 
We propose to use multi-task representation learning for physical systems, and show how it can bridge the gap between the power of neural networks and the interpretability of the model, with minimal computational costs.
% {\color{red} \st{We propose a method aiming to bridge the gap between the power of neural networks and the interpretability of the model, with minimal computational costs.}}

\paragraph{Contributions} In this work, we study the problem of multi-environment learning of physical systems.
{We model the variability of physical systems with a multi-task representation learning architecture that is affine in task-specific parameters.}
% {Building on multi-task representation learning, }
% {\color{red}\st{We propose a novel model-agnostic meta-learning architecture
% that is affine in task-specific parameters, enabling fast training and adaptation. We prove that our architecture}
% }
{By exploiting the structure of the learning problem, we show how this architecture lends itself to multi-environment generalization, with considerably lower cost than complex meta-learning methods.}
Additionally, we show that it
enables identification of physical parameters for linearly parametrized systems, and local identification for arbitrary systems. Our method's generalization abilities and computational speed are experimentally validated on various physical systems and compared with the state of the art. The interpretability of our model is illustrated by applications to physical parameter-induced adaptation and to adaptive control.

% \paragraph{Organization} The remainder of our work is organized as follows.

\todo[inline]{Static systems}

% \paragraph{Organization} The remainder of our work is organized as follows.

\todo[inline]{ANIL, difficult computation even for depth-two networks~\citep{yuksel2023model}.}
\todo[inline]{context supervision, \cite{wang2022meta}}
\todo[inline]{Our approach is model agnostic}

% {From a purely technical standpoint, our approach can be seen as a cheaper alternative to MAML.}

\todo[inline]{Domain adaptation, transfer learning \cite{wang2022generalizing}}

\todo[inline]{For (linearly) structured observations, unstructured models fail}
% ============================================================
\section{Learning from multiple physical environments}

In this section, we present the problem of multi-task learning as it occurs in the physical sciences and we summarize how it can be tackled with deep learning in a meta-learning framework.

% In the learning of physical systems, different tasks correspond to different experimental conditions, also called contexts, under which the system is observed.
% In the case of the pendulum, for instance,

% -----------------------------------------
\subsection{The variability of physical systems}

% \paragraph{Multi-task learning}
In general, a physical system is not fixed from one interaction to the next, as experimental conditions vary, whether in a controlled or uncontrolled way.
From a learning perspective, we assume a meta-dataset~${{D} :=  \cup_{t=1}^T D_t}$ composed of~$T$ datasets, each dataset gathering observations of the physical system under specific experimental conditions.
The goal is to learn a predictor from~$D$ that is robust to task changes, in the sense that when presented a new task, it can learn the underlying function from a few samples~\citep{hospedales2021meta}. Note that in practice the number of tasks~$T$ is typically very limited, owing to the high cost of running physical experiments.

For simplicity, we assume a classical supervised regression setting where~${D_t := \{x_{t}^{(i)},  y_{t}^{(i)} \}_{1\leq i \leq N_t}}$ and the goal is to learn a~$x \mapsto y$ predictor, although the approaches presented generalize to other settings such as trajectory prediction of dynamical systems.
We discuss two physical examples illustrating the need for multi-task learning algorithms, with different degrees of complexity.

%  In order to take the most out of the data, the learning model should accomodate to these variations.
% In a statistical learning approach, it is vital that the algorithm used takes these variations into account and is robust to them.
\begin{example}[Actuated pendulum]
    \label{example:pendulum}
    We begin with the pendulum, one of physics' most famous toy systems. Denoting its inertia and its mass by~$I$ and~$m$ and the applied torque by~$u$, the angle~$q$ obeys
    \begin{equation}
        \label{eq:pendulum}
        I \ddot{q} + m g \sin q
        =
        u.
    \end{equation}
    For example, we may want to learn the action~$y=u$ as a function of the coordinates~${x=(q, \dot{q}, \ddot{q})}$.
    In a data-driven framework, the trajectories collected may show variations in the pendulum parameters: the same equation~\eqref{eq:pendulum} holds true, albeit with different parameters~$m$ and~$I$.
    % A predictor of~$y=u$ as a function of~$x=(q, \dot{q}, \ddot{q})$
    % The goal is to capture a task-agnostic representation of the pendulum using the diversity and the commonalities of the observations from multiple contexts.
    % For instance~$x :=(q, \dot{q}, \ddot{q})$ and $y(x)=u(q, \dot{q}, \ddot{q})$. As we shall see, this linear structure in the inertial parameters is in fact of general application to mechanical systems.
\end{example}
A more complex, non-analytical example is that of learning the solution to a partial differential equation, which is rarely known in closed form and varies strongly according to the boundary conditions.
\begin{example}[Electrostatic potential]
    \label{example:potential}
    The electrostatic potential~$y$ in a space~$\Omega$ devoid of charges solves~Laplace's equation, with boundary conditions
    \begin{equation}
        \label{eq:pde}
        \Delta y  =
        0
        % -\frac{\rho}{\epsilon_0}
        \quad \text{on} \; \Omega, \qquad
        % \big(x \in \Omega \mapsto y(x)\big)
        %  = \text{solve}\big(\partial \Omega, y(\partial \Omega)\big)
        y(x) = b(x)  \quad \text{on} \; \partial \Omega.
        % \big(x \in \Omega \mapsto y(x)\big)
        %  \big(y(x) \big)_{x\in \omega}
    \end{equation}
    A robust data-driven solver should be able to generalize to (at least small) changes of~$\partial \Omega$ and~$b$.
    %     \begin{equation}
    % \begin{aligned}
    %             \Delta y & = -\rho/\epsilon_0 \quad \text{on} \; \Omega,
    %             \\
    %             % \big(x \in \Omega \mapsto y(x)\big)•
    %             %  &= \text{solve}\big(\partial \Omega, y(\partial \Omega)\big)
    %             y(x) &= \bar{y} \quad \text{on} \; \partial \Omega
    %             % \big(x \in \Omega \mapsto y(x)\big)
    %             %  \big(y(x) \big)_{x\in \omega}
    % \end{aligned}
    %     \end{equation}
    % Experiments may differ according to variations in boundary geometry or boundary conditions.
\end{example}

% Depending on the complexity and knowledge of the system, experimental variations can take a more or less simple form, as illustrated by the following examples.

% ---------------------------------------------------------
\subsection{Overview of
% \cancel{\color{red}gradient-based meta}
{multi-environment deep learning}}
\label{section:overview}

Multi-task statistical learning has a long history, and several approaches to this problem have been proposed in the statistics community~\citep{caruana1997multitask}.
We will focus on the meta-learning paradigm~\citep{hospedales2021meta}, which has recently gained considerable importance and whose application to neural nets looks promising given the complexity of physical systems.
We next describe the generic structure of
% {\color{red}\st{gradient-based}}
 meta-learning algorithms for multi-task generalization. The goal is to obtain a~$x \mapsto y$ mapping in the form of a two-fold function~$y\simeq f(x;w)$, where~$w$ is a tunable task-specific weight that models the environment variations.

\paragraph{Learning model}
Given the
% {\color{red}\st{popularity}}
{ learning capabilities}
 of neural networks,
incorporating multi-task generalization into their gradient descent training algorithms is a major challenge. 
Since the seminal paper by~\cite{finn2017model}, several algorithms have been proposed for this purpose, with the common idea of finding a map adapting the weights of the neural network according to task data. A convenient point of view is to introduce a two-fold parametrization of a meta-model~$F(x;\theta,w)$, with a task-agnostic parameter vector~$\theta \in \R^p$ and task-specific weights~$w$~(also called learning contexts). For each task~$t$, the task-specific weight is computed based on some trainable meta-parameters~$\pi$ and the task data currently being processed as~$w_t := {A}(\pi, D_t)$, according to an adaptation rule~$A$ that is differentiable with respect to~$\pi$. The meta-parameters are trained to minimize the meta-loss function aggregated over the tasks, as we will see below.
% The model is then trained to minimize the sum of the task loss functions.
% Given the popularity and expressiveness of neural networks, incorporating multi-task learning into their gradient descent training algorithms is a major challenge. Since Finn's seminal paper, several algorithms have been proposed for this purpose, and here we summarize how they work. A parametric model, which could be as complex as a neural network, is used for learning from multi-environment data. In order to incorporate the contextual structure of the data, the make the model learn what is common to different environments, while leaving it room to learn different contexts, the model parameters are separated into a context and a context-agnostic component.
% The model parameters can be separated into a task-agnostic component~$\theta$ and a context-specific component~$w$, called context.
% The model is denoted~$f(x ; \theta, w)$. In this framework, a meta-learning algorithm is characterized by an parameter adaptation map~${a: (\pi, D_t) \mapsto w_t}$ specifying how the task-specific weights of~$f$ are instantiated from the meta-parameters~$\pi$ and the dataset of task~$t$.
In this formalism, a meta-learning algorithm is determined by the meta-model~$F(x;\theta,w)$ and the adaptation rule~$A$.

We provide examples of recent architectures in~Table~\ref{table:architectures}. In MAML~\citep{finn2017model}, the meta-parameter~$\pi$ is simply~$\theta$ and the adaptation rule is computed as a gradient step in the direction of the task-specific loss improvement, in an inner gradient loop. In~CoDA~\citep{pmlr-v162-kirchmeyer22a}, the meta-parameter~$\pi$ has a dimension growing with the number of tasks $t$ and the adaptation rule is computed directly from the meta-parameters, with task-specific low-dimensional context vectors~$\xi_t \in \R^{d_\xi}$ and a linear hypernetwork~${\Theta \in \R^{p \times d_\xi}}$. Variants of~MAML,~CAVIA~\citep{zintgraf2019fast} and~ANIL~\citep{raghu2019rapid}, fit into this scheme as well and correspond to the restriction of the adaptation inner gradient loop to a predetermined set of the network's weights. This framework also encompasses
% {\color{red}\st{our new architecture~CAMEL}}
{{the CAMEL algorithm}, which we introduce in~Section~\ref{section:model}}.

\paragraph{Meta-training}
The training process is summarized in~Algorithm~\ref{algorithm:meta-training}.
For each task~$t$, the meta-learner computes a task-specific version of the model from the task dataset~$D_t$, defining~
${f_t(x; \pi) : = F(x ;  \theta, A(\pi, D_t))}$.
The error on the dataset~$D_t$ is measured by the task-specific~loss
\vspace{-0.1cm}
\begin{equation}
    \label{eq:task-loss}
    \ell(D_t; \theta, w) =
  \sum\limits_{x, y \, \in D_{t} }\frac{1}{2} \big( F(x ; \theta, w) - y \big)^2.
\end{equation}
Parameters~$\pi$ are trained by gradient descent in order to minimize the regularized meta-loss defined as the aggregation of~$L_t$ and a regularization term~$R(\pi)$:
\vspace{-0.1cm}
\begin{equation}
    \label{eq:meta-loss}
    L(\pi) := \sum\limits_{t=1}^T
    \ell\big(D_t; \theta , w_t( \pi) \big)
    + R(\pi).
\end{equation}
\begin{minipage}[c]{.447\linewidth}
    \begin{algorithm}[H]
        \caption{Gradient-based meta-training}
        \label{algorithm:meta-training}
        \begin{algorithmic}
            \State \textbf{input}
            meta-model~$F(x ; \theta, w)$,
            adaptation rule~$A$,
            initial meta-parameters~$\pi$,
            learning rate~$\eta$,
            task datasets~$D_1, \dots D_T$
            \State \textbf{output} learned meta-parameters~$\bar{\pi}$
            \While{not converged}
            \For{tasks $1 \leq t \leq T$}
            \State  compute $\theta $ from $\pi$
            \State  adapt $w_t :=A(\pi, D_t)$
            \State  compute $\ell\big(D_t; \theta, w_t(\pi)\big)$

            % \State  define $f_t(x; \pi) := F(x ; \theta, w_t)$
            % \State  compute $L_t[f_t(. \, ; \pi)]$,
            % \; as in~\eqref{eq:task-loss}
            % \State  instantiate parameters $(\theta, w_t) = \phi(\pi, D_t)$
            % \State  instantiate model $f_t(x; \pi) = f(x ; \theta, w_t)$
            % \State  compute the task-specific loss $L_t[f_t(x ; \pi)]$ \quad \eqref{eq:task-loss}
            \EndFor
            \State  {compute} $\displaystyle L(\pi)$,\; as in~\eqref{eq:meta-loss}
            % $ = \frac{1}{T} \sum\limits_{t=1}^T L_t[f_t(x;\pi)]$
            \State  {update} $\pi \leftarrow \pi - \eta \nabla L (\pi)$
            \EndWhile
        \end{algorithmic}
    \end{algorithm}
\end{minipage}
\,
\begin{minipage}[c]{.53\linewidth}
    % \begin{table}
    \centering
    \captionof{table}{Structure of various meta-learning models. Here~$h(x;\theta)\in \R$ and~$v(x;\theta) \in \R^r$ denote arbitrary parametric models, such as neural networks; ~``order" stands for differentiation order.
        % For ANIL and~CAVIA, only a restricted number of layers of the neural network are adapted.
    }
    \label{table:architectures}
    \resizebox{\textwidth}{!}{
        \begin{tabular}{c||c|c|c}
            % \hline
                       & MAML                          & CoDA                        & CAMEL
            \\
            \hline
            % \hline
            % \multirow{2}{*}{structure}          & \multicolumn{2}{c|}{parameter-}                                                                               & function-
            % \\
            %                                     & \multicolumn{2}{c|}{additive}                                                                                 & linear
            % \\
            % \hline
            $\pi$      & \multicolumn{1}{c|}{$\theta$} & $\theta, \Theta, \{\xi_t\}$ & $\theta, \{\omega_t\}$
            \\
            \hline
            $\mathrm{dim} (\pi)$
            \rule{0pt}{2.2ex}
                       & $p$                           &
            ${p {+} p{\times} d_\xi {+} d_\xi{\times} T}$
                       & $p {+} r {\times} T$
            \\
            \hline
            $\mathrm{dim}(w)$
            % \rule{0pt}{2.2ex}    
                       &
            \multicolumn{2}{c|}
            {  $p$}
                       &
            {$r$ }
            \\
            \hline
            ${A}(\pi, D_t)$
            \rule{0pt}{2.2ex}
                       & $-\alpha \nabla_\theta L_t$   & $ \Theta\xi_t$              & $\omega_t $
            \\
            \hline
            % \multirow{2}{*}
            {$F(x ; \theta, w)$}
            \rule{0pt}{2.2ex}
                       &
            \multicolumn{2}{c|}
            {  $h(x; \theta + w )$}
            %  & {  $h(x; \theta + \Theta w )$}
                       &
            %    \multirow{2}{*}
            {$\transp{w} v(x ; \theta)$ }
            % \\
            % % \cline{2-3}
            %                                     & \multicolumn{2}{c|}{$\delta \theta = w$                      \quad                $\delta\theta = \Theta w$ } &
            \\
            \hline
            training   & \multirow{2}{*}{2}            & \multirow{2}{*}{1}          & \multirow{2}{*}{1}
            \\
            order      &                               &                             &
            \\
            \hline
            adaptation & \multirow{2}{*}{1}            & \multirow{2}{*}{1}          & \multirow{2}{*}{0}
            \\
            order      &                               &                             &
            \\
            % $R(\pi)$                               & $0$ & $\Vert \Theta \Vert ^2$
            % & $ \Vert w_t\Vert^2$               
            % \\
            \hline
        \end{tabular}
    }
    %     % \label{table:structure}
    % \end{table}
\end{minipage}

\paragraph{Test-time adaptation}
Once training is complete, the trained meta-parameters~$\bar{\pi}$ define a tunable model~${f(x;w) := F(x; \bar{\theta}, w)}$, where~$\bar{\theta}$ is the trained task-agnostic parameter vector.
At test time, the trained meta-model is presented with a dataset~$D_{T+1}$ consisting of few samples~(or shots) from a new task. Using this adaptation data,~$\bar{\theta}$ is frozen and
% and the learned meta-parameters~$\bar{\pi}$,
% the task-agnostic component~$\bar{\theta}$ of the meta-model is frozen, and
the task-specific weight~$w$ is tuned~(possibly in a constrained set) by minimizing the prediction error on the adaptation dataset:
\begin{equation}
    \label{eq:adaptation_problem}
    w_{T+1} \in \underset{w}{\mathrm{argmin}}
    % \;  L_{T+1}(\bar{\theta}, w ;  D_{T+1}).
    \; \ell\big(D_{T+1}; \bar{\theta} , w \big).
% \ell\big(D_t; \theta , w_t( \pi) \big)
\end{equation}
In all the above approaches, this minimization is performed by gradient descent.
The resulting adapted predictor is defined as~$F (x ;\bar{\theta}, w_{T+1})$.
The meta-learning algorithm is then evaluated by the performance of the adapted predictor on new samples from task~$T+1$.
% This test process can be averaged over the task distribution.

{
    % \color{red}
    \paragraph{Computational cost}
    % \st{
    The inner-loop gradient-based adaptation used in~MAML and its variants suffers from the computational cost of second-order optimization, since Hessian-vector products are computed in numbers proportional to the number of tasks.
    Furthermore, the cost of gradient-based adaptation at test time can also be crucial, especially for real-time applications where the trained model must be adapted at high frequency.
}
% }
% {
%     \paragraph{Computational cost}
%     The inner-loop operations in~$A(\pi, D_t)$ are evaluated and differentiated through in a number proportional to~$N_t$, the computational cost of adaptation may become prohibitive. For instance, the gradient-based inner-loop adaptation used in~MAML and its variants suffers from the computational cost of second-order optimization, since Hessian-vector products are computed in numbers proportional to~$N_t$.
%     In the case of learning physical systems, the number of measurements~$N_t$ may be very large 
%     and the number~$T$ of available tasks at training is typically very limited. 
%     We are therefore in a different regime from that of few-shot image classification, for example, where on the contrary
%     % ~$T$ is large (and~\eqref{eq:meta-loss} is never summed over all the tasks in practice) and
%     ~$N_t$ is small~\citep{wang2022meta}.
%     This has importance consequences regarding the cost of training (Algorithm~\ref{algorithm:meta-training}) in our setting
%     Furthermore, the cost of gradient-based adaptation at test time can also be crucial, especially for real-time applications where the trained model must be adapted at high frequency.
% }

% ============================================================
\section{
    %  \cancel{\color{red}CAMEL: }
    Context-Affine Multi-Environment Learning}
\label{section:model}
{
Physical systems often have a particular structure in the form of mathematical models and equations. The general idea behind model-based machine learning is to exploit the available structure to increase learning performance and minimize computational costs~\citep{karniadakis2021physics}. With this in mind, we adopt in this section a %seemingly 
simpler architecture than those shown above, and show how it lends itself particularly well to learning physical systems.

\paragraph{Problem structure}
We note that many equations in physics exhibit an affine task dependence, since the varying physical parameters often are linear coefficients~(as we see in~Example~\ref{example:pendulum}, and we shall further explain in~Section~\ref{section:identification}). By incorporating this same structure and hence mimicking physical equations, the model should be well-suited for learning them and for interpreting the physical parameters.
% Additionally, we note that the affine task dependence found in~\eqref{eq:model_expansion} is exhibited in many equations in physics, where the varying physical parameters are linear coefficients~(as we see in~Example~\ref{example:pendulum}, and we shall further explain in~Section~\ref{section:identification}). Following these intuitions
% }
% {\color{red}\st{We go one step further and introduce a
% new architecture where the dependence on task-specific parameters is explicitly affine, which we
% call Context-Affine Multi-Environment Learning (CAMEL).}}
% and leveraging the assumption that the number of tasks~$T$ is small,
    % {\color{red} \st{introduce a new architecture }}
        Following these intuitions, we propose to learn multi-environment physical systems with affine task-specific context parameters.
}

\begin{definition}[
    % {\color{red}\st{CAMEL metamodel}}
    {Context-affine multi-task learning}]
    \label{definition:task-linear}
    The prediction is modeled as an affine function of low-dimensional task-specific weights~$w \in \R^r$ with a task-agnostic feature map~${v(x ;\theta) \in \R^r}$ and a task-agnostic bias~$c(x;\theta) \in \R$:
    \begin{equation}
        \label{eq:task-linear_model}
        F(x ; \theta, w) = c(x; \theta) + \transp{w}v(x; \theta).
        %  \qquad \in \R.
    \end{equation}
    % {Note that this structure relates to multi-task representation learning with hard parameter sharing~\citep{caruana1997multitask}.}
    The dimension~$r$ of the task weight must be chosen carefully. It must be larger than the estimated number of physical parameters varying from task to task but smaller than the number of training tasks, so as to observe the function~$v$ projected over a sufficient number of directions. During training, the task-specific weights are directly trained as meta-parameters along with the shared parameter vector:~$\pi = (\theta, \omega_1 \dots, \omega_T)$ and $w_t = A(\pi, D_t) = \omega_t$.
The meta-parameters are jointly trained by gradient descent as in~Algorithm~\ref{algorithm:meta-training}.
At test time, the minimization problem of adaptation~\eqref{eq:adaptation_problem} reduces to ordinary least squares.
\end{definition}

{
The architecture introduced in Definition~\ref{definition:task-linear} is equivalent to multi-task representation learning with hard parameter sharing~\cite{ruder2017overview} and is proposed as a meta-learning algorithm in~\citep{wang2021bridging}
%  where it is named ``Multi-Task Learning''. Since multi-task learning is also a more general statistical learning paradigm~\citep{caruana1997multitask},
  We will refer to it
%   the model of Definition~\ref{definition:task-linear} 
  in our physical system framework as~Context-Affine Multi-Environment Learning~(CAMEL).
In this work, we show that~CAMEL is particularly relevant for learning physical systems.
%, for which it is a novel learning architecture. 
Table~\ref{table:architectures} compares~CAMEL with the meta-learning algorithms described above.
}
% {
% \paragraph{Specificity}
% It is worth noting that adapting linear task-specific weights is not a new idea in the deep learning community. For instance, ANIL adapts the last layer of a neural network in an inner-loop~(see Table~\ref{table:adaptation}), and another adaptation rule based on closed-form Ridge regression is proposed in~\citep{bertinetto2019meta}. However, the crucial difference lies in the fact that, here, the task-specific parameters~$(\omega_t)_{t=1}^T$ are not computed at each epoch, but they are stored and jointly trained. This algorithmic structure relies on the central assumption that~$T$ is small enough for the~$\omega_t$ to be stored, which is specific to our multi-environment learning framework, where few experiments of a physical system are observed. It is not realistic in few-shot image classification for instance, where~$T$ is possibly in the millions~\citep{wang2021bridging, hospedales2021meta}. 
% }

\paragraph{Computational benefits}
As the task weights~$(\omega_t)_{t=1}^T$ are kept in memory during training instead of being computed in an inner loop,~CAMEL can be trained at minimal computational cost. In particular, it does not need to compute Hessian-vector products as in~MAML,
{
or to propagate gradients through matrix inversions as in~\citep{bertinetto2019meta}. The latter operations can be prohibitively costly in our physical modeling framework, where the number of data points~$N_t$ is large~(it is typically the size of a high-resolution sampling grid, or the number of samples in a trajectory).
Adaptation at test time is also computationally inexpensive since ordinary least squares guarantees a unique solution in closed form, as long as the number of samples exceeds the dimension $r$ of the task weight.}
For real-time applications, the online least-squares formula~\citep{kushner2003stochastic} ensures adaptation with minimal memory and compute requirements, whereas gradient-based adaptation {(as in CoDA or in MAML)} can be excessively slow.

\paragraph{Applicability}
The { meta-learning} models described in~Section~\ref{section:overview} seek to learn multi-task data from a complex parametric model (typically a neural network), making the structural assumption that the weights vary slightly around a central value in parameter space:~${f_t(x; {\pi}) = h(x ; \theta_0 + \delta \theta_t)}$, with~$\Vert \delta \theta \Vert \ll \Vert \theta_0 \Vert$.  Extending this reasoning, the model should be close to its linear approximation:
\begin{equation}
    \label{eq:model_expansion}
    h(x ; \theta_0 + \delta \theta_t) \simeq h(x; \theta_0) +
    % \partial {\partial g}/{\partial \theta} \delta \theta_t
    \transp{\delta \theta_t} \nabla h(x; \theta_0),
\end{equation}
where we observe that the output is an affine function of the task-specific component~$\delta \theta_t$. We believe that~\eqref{eq:model_expansion} explains the observation that~MAML mainly adapts the last layer of the neural network~\citep{raghu2019rapid}.
In~Definition~\ref{definition:task-linear},~$v$ and~$c$ are arbitrary parametric models, which can be as complex as a deep neural network and are trained to learn a representation that is linear in the task weights. Following~\eqref{eq:model_expansion}, we expect~CAMEL's expressivity to be of the same order as that of more complex architectures, with~$c(x;\theta)$,~$w_t$ and~$v(x; \theta)$ playing the roles of~$h(x;\theta_0)$,~$\delta\theta_t$ and~$\nabla h(x; \theta)$ respectively.     
Another key advantage of~CAMEL is the interpretability of the model, which we describe next.

% The task-linear function~\eqref{eq:task-linear_model} can be generalized to multivariate observations~$y\in \R^m$ with the parametrization~${F(x ;\theta,  w) = V(x; \theta)\times { w}} + c(x; \theta)$, and~${V(x; \theta) \in \R^{m \times r}}$.

% \paragraph{System rank} How to choose~$r$? This number should be as least equal to the number of independent parameters that vary in the data.
% The learning rank $r$ should be greater or equal than~$n$ and smaller than the number of tasks, such that the learner observes the function~$v$ projected on a sufficinent number of vectorial lines.

%  Instead, computing them in terms of gradients in an inner adaptation loop as in~\citep{finn2017model} induces second-order gradients terms when differentiating through the outer loop.

% Other meta-learning approaches, in contrast, require a number of gradient descent adaptation steps and carefully choosing a learning rate.
% The linear structure of our model allows us to solve~\eqref{eq:adaptation_problem} by ordinary least squares.
% Gradient free adaptation

% Another key advantage of our approach is the interpretability of the model, which we describe next.

% ============================================================
\section{Interpretability and system identification}
\label{section:identification}

% In the physical sciences, considerable importance is attached to the interpretability of the learning algorithm.
The observations of a physical system are often known to depend on certain well-identified physical quantities that may be of critical importance in the scientific process. When modeling the system in a data-driven approach, it is desirable for the trained model parameters to be interpretable in terms of these physical quantities~\citep{karniadakis2021physics}, thus ensuring controlled and explainable learning~\citep{e23010018}. We here focus on the identification of task-varying physical parameters, which raises the question of the identifiability of the learned task-specific weights.
System identification and model identifiability are key issues when learning a system~\citep{ljung1998system}.
% Parameter identification for physical systems has a long history and is a difficult problem in general~\citep{nelles2001nonlinear}.
Although deep neural networks are becoming increasingly popular for modeling physical systems, their complex structure makes them impractical for parameter identification in general~\citep{nelles2001nonlinear}.
% The first reason is that neural network parameters are non-identifiable in general, in the sense that there exist different parameter vectors yielding the same output~\citep{pourzanjani2017improving}. As a consequence, there is an ambiguity when speaking of the parameters of a neural network model. Second, there is anyway little hope giving a physical meaning to the learned parameters of a deep neural network. Indeed, the sequence of multiplications and  nonlinearities operated on the inputs in a neural model is motivated by the computational nature of automatic differentiation rather than by structural assumptions on the observed system~\citep{lecun2015deep}. As contexts are encoded deep in the neural network, it is hard to make sense of them and neural networks operate like a black-box model. 
% In contrast, the linear parametrizations that we have introduced in~Section~\ref{section:model} are more easily amenable to parameter identification~\citep{ljung1998system}.

\paragraph{Physical context identification}
In mathematical terms, the observed output~$y$ is considered as an unknown function~$f_\star(x ; \varphi)$ of the input and a physical context vector~$\varphi \in \R^n$, gathering the parameters of the system. In our multi-environment setting, each task is defined by a vector~$\varphi_t$
as~$y(x; t) = f_\star(x, \varphi_t)$.
% \begin{equation}
%     y(x; t) = f_\star(x, \varphi_t).
% \end{equation}
At test time, a new environment corresponds to an unknown underlying physical context~$\varphi_{T+1}$.
While adaptation consists in minimizing the prediction error on the data as in~\eqref{eq:adaptation_problem}, the interpretation goes further and seeks to identify~$\varphi_{T+1}$.
This means mapping the learned task-specific weights~$w$ to the physical contexts~$\varphi$, \textit{i.e.} learning an estimator~$\hat{\varphi}: w \mapsto \varphi$ using the training data and the trained model.
Assuming that the physical parameters of the training data~$\{ \varphi_t \}$ are known,
% \begin{equation}
%     f_\star(x^{(i)}_t, \varphi_t) \simeq f(x^{(i)}_t, w_t)
% \end{equation}
this can be viewed as a regression problem with~$T$ samples, where~$\hat{\varphi}$ is trained to predict~$\varphi_t$ from weights~$w_t$ learned on the training meta-dataset.

% ---------------------------------------
\subsection{Linearly parametrized systems}

% \todo[inline]{Linear in parameters but the learned function is unknown}

We are primarily interested in the case where the physical parameters are known to intervene linearly in the system equation, as
% Many physical systems can be modeled mathematically  as a sum of various functions representing different physical contributions, with scalar coefficients that may vary across experiments.
% The core observation underlying our framework is the following.
% Regardless of the experiment, the mathematical functions that describe the system are the same, and what varies across various experiments is only a vector of coefficients. 
% Under this assumption, the family of target functions has the form
\begin{equation}
    \label{eq:task-linear_system}
    f_\star(x;\varphi):= \kappa(x) + \transp{\varphi}\nu(x), \quad \nu(x) \in \R^n.
\end{equation}
% \begin{sloppypar}
This class of systems is of crucial importance: although simple, it covers a large number of problems of interest, as the following examples illustrate. Furthermore, it can apply locally to more general system, as we shall see later.

%     \begin{example}[Electrostatic field]
%         \label{example:dipole}
%         An electromagnetic field is typically the sum of several contributions (\textit{e.g.} an ambient field and an object-specific field), which are neither known nor controlled by the experimenter and hence vary from an experimenter to another. However, their mathematical expressions across tasks are known to have a common structure, up to some linear coefficients. For instance, we consider the potential of an electrostatic dipole moment in the presence of an ambient field:
%         \begin{equation}
%             % $
%             f_\star(x; \varphi) =  \frac{1}{2\pi \epsilon_0}\frac{p x_1}{ (x_1^2 + x_2^2)^{3/2}} - U x_1 ,
%             % $,
%         \end{equation}
%         This physical system has the form of~\eqref{eq:task-linear_system} with the system parameters~${\varphi=(U, p)}$ and features~${v(x) = \big(x_1, x_1/(2\pi \epsilon_0(x_1^2 + x_2^2)^{3/2})\big)}$. The corresponding electric field is given by~${E(x) = -\nabla f_\star(x) \in \R^2}$~\citep{haus1989electromagnetic} .
%     \end{example}
% \end{sloppypar}
\begin{example}[Electric point charges]
    \label{example:charges}
    Point charges are a particular case of~Example~\ref{example:potential} with point boundary conditions, proportional to the charges~$\varphi = (\varphi^{(1)}, \dots, \varphi^{(n)})$. The resulting field can be computed using~Coulomb's law and is proportional to these charges:
    % \begin{equation}
    $f_\star(x; \varphi) = \transp{\varphi}\nu(x)$,
    with~${\nu(x)\propto ({1}/{\Vert x-x^{(j)}\Vert})_j}$. Although the solution is known in closed form, this example can illustrate more complex problems where an analytical solution is out of reach (and hence~$\nu$ is unknown) but the linear dependence on certain well-identified parameters is postulated or known.
    % with~$v_j(x)= \frac{1}{4\pi \varepsilon_0} \frac{1}{\Vert x-x_j\Vert}$.
    %     $
    % f_\star(x; \varphi) = \sum\limits_{j=1}^n \varphi_j v_j(x), \quad \text{with} \quad v_j(x)= \frac{1}{4\pi \varepsilon_0} \frac{1}{\Vert x-x_j\Vert}.
    % $
    % \end{equation}
\end{example}
\begin{example}[Inverse dynamics in robotics]
    \label{example:inverse-dynamics}
    The Euler-Lagrange formulation for the rigid body dynamics has the form
    \begin{equation}
        \label{eq:manipulator}
        M(q)\ddot{q} + C(q, \dot{q})\dot{q} +g(q) = B u,
    \end{equation}
    where $q$ is the generalized coordinate vector,~$M$ is the mass matrix,~$C$ is the Coriolis force matrix,~$g(q)$ is the gravity vector and the matrix~$B$ maps the input~$u$ into generalized forces~\citep{underactuated}. It can be shown that~\eqref{eq:manipulator} is linear with respect to the system's dynamic parameters~\citep{nguyen2010using}, and hence takes the form of~\eqref{eq:task-linear_system} for scalar controls. A simple, yet illustrative system with this structure is the actuated pendulum~\eqref{eq:pendulum}, where it is clear that the equation
    % ~$u(q, \dot{q}, \ddot{q})$
    is linear in the inertial parameters~$I$ and~$m$. The inverse dynamics equation can be used for trajectory tracking~\citep{spong2020robot}, as it predicts~$u$ from a target trajectory~$\{ q(s) \}$~(see~Appendix~\ref{appendix:inverse-dynamics-control}).
    % \begin{equation}
    %     I \ddot{q} + m g \sin q
    %     % + \alpha \dot{q} 
    %     =
    %     % - \alpha \dot{q} +
    %     u.
    % \end{equation}
    % If several pendulums are observed in the training data, the system parameters~$\varphi := (I, mg)$ of dimension~$r=2$ may vary across the datasets, but the torque~$u$ remains a linear function of the same feature map~${v(x) = (\ddot{q}, \sin q)}$. The goal is to learn both the task-agnostic and the task-specific components from observations of two different environments~$y = \transp{\varphi}v(x)$ and~$y' = \transp{\varphi'}v(x)$.
\end{example}
% -------------------------------------------------
\subsection{Locally linear physical contexts}

In the absence of prior knowledge about the system under study, the most reasonable structural assumption for multi-task data is to postulate small variations in the system parameter:~${\varphi = \varphi_0 + \delta \varphi}$. The learned function can then be expanded and found to be locally linear in physical contexts:
\begin{equation}
    \label{eq:locally-linear}
    f_\star(x;\varphi) \simeq f_\star(x;\varphi_0) + \transp{\delta \varphi}\nabla f_\star(x ; \varphi_0),
\end{equation}
which has the form~\eqref{eq:task-linear_system} with~$\kappa(x) = f_\star(x; \varphi_0)$ and~$\nu(x) = \nabla f_\star(x ; \varphi_0)$.
% Given that the physical system is characterized by the feature map~$\nu$ that is common across the environments, an environment of the system is entirely determined by a vector of physical parameters~$w$.
\begin{example}[Identification of boundary perturbations]
    \label{example:local-boundary}
    For a general boundary value problem such as~\eqref{eq:pde}, we may assume that the boundary conditions~$\partial \Omega(\varphi), b(x, \varphi)$ vary smoothly according to parameters~$\varphi$ (such as angles or displacements). If these variations are small and the problem is sufficiently regular, the resulting solution~$f_\star(x, \varphi)$ can be reasonably approximated by~\eqref{eq:locally-linear}.
\end{example}

% ---------------------------------------
\subsection{System identification with CAMEL}
\label{section:camel_identification}

We now study the problem of system identification under the assumption of parameter linearity~\eqref{eq:task-linear_system} using the~CAMEL metamodel~\eqref{eq:task-linear_model}.
%  trained with~Algorithm~\ref{algorithm:meta-training}, yielding weights~${w}_1, \dots, {w}_T$ and a trained feature map~${{v}(x):={v}(x,\bar{\theta})}$. 
We study the identifiability of the model and therefore investigate the vanishing training loss limit, with~$c=\kappa=0$ for simplicity, yielding
\begin{equation}
    \transp{\omega_t} v(x^{(i)}_t) = \transp{\varphi_t} \nu(x^{(i)}_t) \quad \text{for all} \quad 1\leq t \leq T, \; 1 \leq i \leq N_t.
\end{equation}
% The common form of the trained model~\eqref{eq:task-linear_model} and the observations~\eqref{eq:task-linear_system} add structure to the identification problem of learning~$\varphi$ from~$w$, which would otherwise be difficult to study.
\paragraph{Identifiability}
Posed as it is,
%  with the only supervision being the data,
we can easily see that the~physical parameters~$\varphi_t$ are not directly identifiable. Indeed, for any~$P\in\mathrm{GL}_r(\R)$, the weights~$\omega$ and the feature map~$v$ produce the same data as the weights~$\omega':=\transp{P}\omega$ and the feature map~${v' = P^{-1}v}$, since~$\transp{\omega}v=\transp{\omega} PP^{-1}v$. This problem is related to that of identification in matrix factorization~(see for example~\cite{fu2018identifiability}).
% Therefore,~Algorithm~\ref{algorithm:meta-training} has no reason to converge to the ground truth values~$\varphi_t$ and~$\nu$ because any other~$w'$ and~$v'$ yield the same training loss.
Now that we have recognized this symmetry of the problem, we can ask whether it characterizes the solutions found by~CAMEL. The following result provides a positive answer.
%   in the vanishing training loss limit. 

\begin{proposition}
    \label{proposition:identification}
    Assume that the training points are uniform across tasks:~$x_{t}^{(i)} = x^{(i)}$, and~${N_t=N}$ for all~$1 \leq t \leq T$ and~$1 \leq i \leq N${, with ~$n\leq r<N, T$. Assume that both sets~$\{ \nu(x^{(i)}) \}$ and~$\{ \varphi_t \}$} span~$\R^{n}$. In the limit of a vanishing training loss~$L(\pi)=0$, the trained meta-parameters recover the parameters of the system up to a linear transform: there exist~$P, Q\in \R^{n \times r}$ such that~$\varphi_t = P \omega_t$ for all training task~$t$ and~${{\nu}(x^{(i)}) = Q v(x^{(i)})}$ for all~$1 \leq i \leq N$. Additionally,~$Q\transp{P} = I_n$.
    % there exists~${P\in\R^{r^\star \times r}}$ such that$
    %  for all training data point~$x$, with~$P =: W_\star^+\bar{W}\in\R^{n \times r}$.
    % ~${\omega_t = \transp{P}  w^\star_t}$ and
    % If~${r=r^\star}$, then~$P{\in \mathrm{GL}_r(\R)}$ and~$\transp{Q}=P^{-1}$.
\end{proposition}
% \todo[inline]{$\varepsilon$ version?}
% \todo[inline]{gradient flow?}
% Assuming fixed training points~$\{ x_{t}^{(i)} \} = \{ x_i\}$ across the task datasets, and applying~Lemma~\ref{lemma:symmetry} to the training features $V:=\big(v(x_1; \theta), \dots, v(x_N ; \theta)\big)$ and parameters~${W := (w_1,\dots, w_T)}$, and $V'$ and~$W'$ the corresponding ground truths, we obtain the following result.
{
A proof is provided in~Appendix~\ref{appendix:proof}, along with the case~$c \neq \kappa$.
}
Proposition~\ref{proposition:identification} shows that CAMEL learns a meaningful representation of the system's features instead of overfitting the examples from the training tasks. Remarkably, the relationship between the learned weights and the system parameters is linear and can be estimated using ordinary least squares
\begin{equation}
    \label{eq:linear_identification}
    \hat{\varphi}(\omega) = \hat{P} \omega, \quad \hat{P} \in \underset{P \in \R^{n \times r}}{\mathrm{argmin}} \; \frac{1}{2}\sum\limits_{t=1}^T  \Vert P \omega_t - \varphi_t \Vert_2^2.
\end{equation}
Although the relationship between the model and the system, in general, is likely to be complex, especially when deep neural networks are used, the structure of our model and the linear physical contexts enable the derivation of the problem symmetries and the computation of an estimator of the physical parameters.
For black-box meta-learning architectures, exhibiting the symmetries in model parameters and computing an identification map seems out of reach, as the number of available tasks~$T$ can be very limited in practice~\citep{pourzanjani2017improving}.
% other learned predictor is unlikely to be accurate in general, as~$T$ may be very limited and the model and the system may be highly nonlinear. In this section, we will see how~CAMEL's structure allows for system identification in the simple setting of linear and linearized physical contexts.
\vspace{-.1cm}
\paragraph{Zero-shot adaptation}
Looking at the problem from another angle,~Proposition~\ref{proposition:identification} also shows that~$\omega$ can be estimated linearly as a function of~$\varphi$, at least when~$r=n$~(which ensures that~$P$ is nonsingular). Computing an estimator of~$\omega$ as a function of~$\varphi$ with the inverse regression to~\eqref{eq:linear_identification} enables a zero-shot~(or physical parameter-induced) adaptation scenario: when an estimate of the physical parameters of the new environment is known a priori, a value for the model weights can be inferred. We call this adaptation method~$\varphi$-CAMEL.

% \todo[inline]{zero-shot adapatation}
% The feature map can then be inferred from the parameters~$W_\star$ of the training data:
% \begin{equation}
%     \hat{v}(x) = P \bar{v}(x ; \bar{\theta}) \quad \text{with} \quad P =: W_\star^+\bar{W}.
% \end{equation}

% \begin{lemma}
% \label{lemma:products}
% Let~$1 r\leq r'<T,N$. If $W, W'\in \R^{T \times r}$ and~$V, V'\in \R^{T \times r}$ satisfy
% $W\transp{V} = W'\transp{V'}$
% in~$\R^{T\times N}$. Assume that~$W$,~$V$ and~$V'$ are of full rank.
% Then~$V=V'\transp{P}$ with~${P:=W^+{W'} \in \R^{r \times r'}}$ and
% \end{lemma}
% ============================================================
\section{Experimenting on physical systems}
\label{section:experiments}

The architecture that we have presented is expected to adapt efficiently to the prediction of new environments, and identify (locally or globally) their physical parameters, as shown in~Section~\ref{section:identification}.
In this section, we validate these statements experimentally on various physical systems: Sections~\ref{section:charges} and~\ref{section:control} deal with systems with linear parameters~(as in~\eqref{eq:task-linear_system}), on which we evaluate the interpretability of the algorithms. We then examine a non-analytical, general system in~Section~\ref{section:capacitor}. We compare the performances of~CAMEL and its zero-shot adaptation version~$\varphi$-CAMEL introduced in~Section~\ref{section:camel_identification} with state-of-the-art meta-learning algorithms. Our code and demonstration material are available at~\url{https://github.com/MB-29/CAMEL}.

\paragraph{Baselines}
We have implemented the MAML algorithm of~\cite{finn2017model}, and its ANIL variant~\citep{raghu2019rapid}, which is computationally lighter and more suitable for learning linearly parametrized systems (according to observation~\eqref{eq:model_expansion}). We have also adapted the $\ell_1$-CoDA architecture of~\cite{pmlr-v162-kirchmeyer22a} for supervised learning~(originally designed for time series prediction). In all our experiments, the different meta-models share the same underlying neural network architecture, with the last layer of size~$r \gtrsim \mathrm{dim}(\varphi)$. Additional details can be found in~Appendix~\ref{appendix:experiments}.
The linear regressor computed for~CAMEL in~\eqref{eq:linear_identification} is computed after training for all architectures with their trained weights~$w_t$, and is available at test time for identification.

% ------------------------------------------------------------
\subsection{Interpretable learning of an electric point charge system}
\label{section:charges}

\begin{figure}
    \centering
    \begin{subfigure}{.48\linewidth}
        \includegraphics[width=\linewidth]{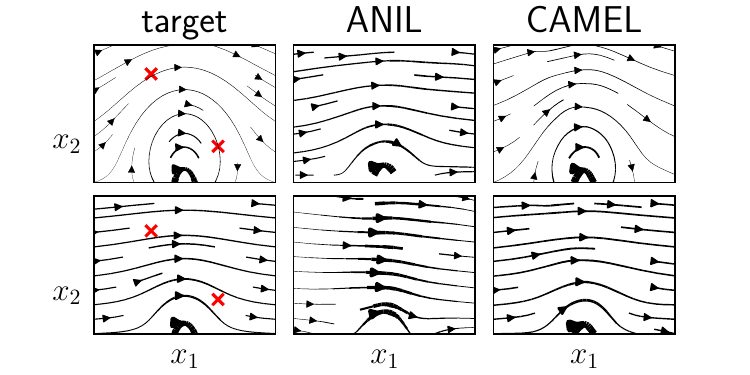}
    \end{subfigure}
    \begin{subfigure}{.48\linewidth}
        \includegraphics[width=\linewidth]{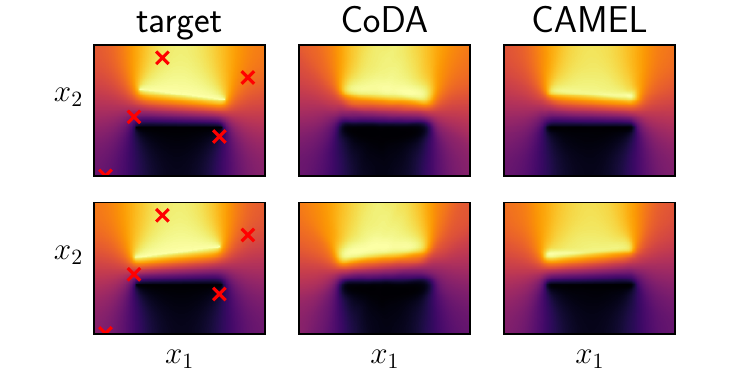}
    \end{subfigure}
    \caption{Few-shot adaptation on two out-of-domain environments of the point charge system in a dipolar setting~\textbf{(left)} and the capacitor~\textbf{(right)}. The adaptation points are represented by the {\color{red} $\times$} symbols. The vector fields are derived from the learned potential fields using automatic differentiation.
    }
    \label{fig:field}

\end{figure}
\begin{wrapfigure}{r}{.25\linewidth}
    \centering
    \vspace{-.65cm}
    \includegraphics[width=\linewidth, trim={.6cm .65cm .5cm .55cm}
        ,clip]{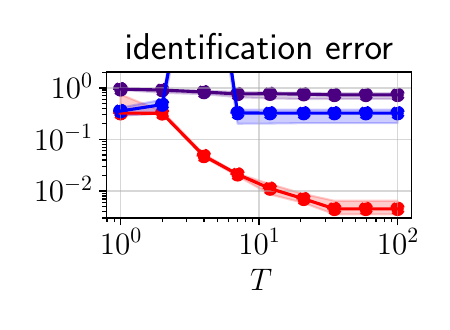}
    \par
    \includegraphics[width=\linewidth, trim={.0cm .03cm .05cm .01cm}
        ,clip]{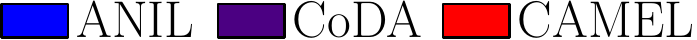}
    \vspace{-.4cm}
    \caption{Average relative error for the point charge identification.}
    \vspace{-.2cm}
    \label{fig:charges_id}
\end{wrapfigure}
As a first illustration of multi-environment learning, we are interested in a data-driven approach to electrostatics, where the experimenter has no knowledge of the theoretical laws (Maxwell's equations, as in~Example~\ref{example:potential}) of the system under study. The electrostatic potential is measured at various points in space, under different experimental conditions. The observations collected are then used to train a meta-learning model to predict the electrostatic field from new experiments, based on very limited data. We start with the toy system described in~Example~\ref{example:charges}, which provides a qualitative illustration of the behavior of various learning algorithms:~$n=3$ point charges placed in the plane at fixed locations. This experiment is repeated with varying charges~$\varphi \in \R^3$.
% ., which and are linear physical parameters of the electrostatic field.
% Able to handle massive amounts of data for physical modeling, artificial intelligence is playing an increasingly important role in scientific discovery~\citep{wang2023scientific}. We present here an illustration in our interpretable meta-learning framework. The ability to generalize to new environments, combined with the interpretability of physical contexts, enables our architecture to design new experimental configurations never before observed by the model.
% Using this additional information,  the experimenter hopes to be able to interpret the predictions of the learning algorithm in terms of the electrical charges and thus envisage new experiments.    

\paragraph{Results} For this system with linear physical parameters, CAMEL outperforms other baselines and can predict the electrostatic field with few shots, as shown in~Figure~\ref{fig:field} and~Table~\ref{table:adaptation}~(5-shot adaptation). Figure~\ref{fig:charges_id} shows the identification error over~30 random test environments with standard deviations, as a function of the number of training tasks. Thanks to the sample complexity of linear regression,~CAMEL accurately identifies system charges, achieving less than~$1\%$ relative error with~10 training tasks. The resulting zero-shot model adapts to the new environment with great precision. We discuss its applicability to scientific discovery in~Appendix~\ref{appendix:discovery}.

% \paragraph{Scientific interpretability}
% With a good physical intuition (acquired through previous experiments, for example), the experimenter assumes that the field depends linearly on the charges, which he has also been able to estimate independently prior to the manipulations.
% Once trained on multiple environments, the interpretability of our model allows it to design new experimental conditions by simply changing the value of the context.
% Understanding this four-charge system is of great practical interest, as it forms the basis of Penning's ion trapping device~\cite{kretzschmar1991particle}.

% In such a situation, the interpretability of statistical parameters is of great value. For example, the experimenter can now model a new experimental setup that has never been observed before, because the load values are exactly identified in the equation. He can then change them freely in the model and observe the new potential field.

% ------------------------------------------------------------
\subsection{Multi-task reinforcement learning and online system identification}
\label{section:control}

Another scientific field in which our theoretical framework can be applied is multi-task reinforcement learning, in which a control policy is learned using data from multiple environments of one system~\citep{electronics9091363}. We saw in~Example~\ref{example:inverse-dynamics} that robot joints obey the inverse dynamics equation, which turns out to be linear in the robot's inertial parameters. Consequently, our architecture lends itself well to the statistical learning of this equation from multiple environment data, as well as to the identification of the dynamic parameters. We may then exploit the learned model of the dynamics to perform adaptive inverse dynamics control~(see~Appendix~\ref{appendix:control}) of robots with unknown parameters, and learn the parameters simultaneously.

% \todo[inline]{Learning from trajectories is challenging}

\paragraph{Systems}
We experiment with systems of increasing complexity, starting with 2D simulated systems:~cartpole and acrobot. To make them more realistic, we add friction in their dynamics. The analytical equation~\eqref{example:inverse-dynamics} is hence inaccurate, which motivates the use of a data-driven learning method. We then experiment on the simulated 6-degree-of-freedom robot Upkie~(Figure~\ref{fig:upkie}), for which~\eqref{eq:manipulator} is unknown and the wheel torque is learned from the ground position and the joint angles.

\paragraph{Experimental setup}
Learning algorithms are trained on trajectories (a more challenging setting than uniformly spaced data) obtained from multiple system environments. At test time, a new environment is instantiated and the model is adapted from a trajectory of few observations. The resulting adapted model is then used to predict control values for the rest of the trajectory. For the carptole and the robot arm, the predicted values are used to track a reference trajectory using inverse dynamics control. For~Upkie, we could not directly use the predicted controls for actuation, but we compare the open-loop predictions with the executed control law. The target motions are swing-up trajectories for the cartpole and the arm, and a 0.5m displacement for~Upkie. Since~Upkie is a very unstable system, it is controlled in a 200Hz model predictive control loop~\citep{rawlings2000tutorial}.
% ~(see~Appendix~\ref{appendix:control}).

\vspace{-.1cm}
\paragraph{Online adaptive control} We also investigate a challenging time-varying dynamics setting where the inertial parameters of the system change abruptly at a given time. This scenario is very common in real life and requires the development of control algorithms robust to these changes and fast enough to be adaptive~\citep{aastrom2013adaptive}. In our case, we double the mass of the cart in the cartpole system, and we quadruple the mass of~Upkie's torso. The learning models adapt their task weights online and adjust their control prediction. In an application to parameter identification, we also compute the estimated values of the varying parameter over time.

\begin{wrapfigure}{r}{.125\linewidth}
    {%
        \setlength{\fboxsep}{0pt}%
        \setlength{\fboxrule}{1pt}%
        \fbox{\includegraphics[width=\linewidth, trim={.8cm .8cm .8cm .8cm}, clip]{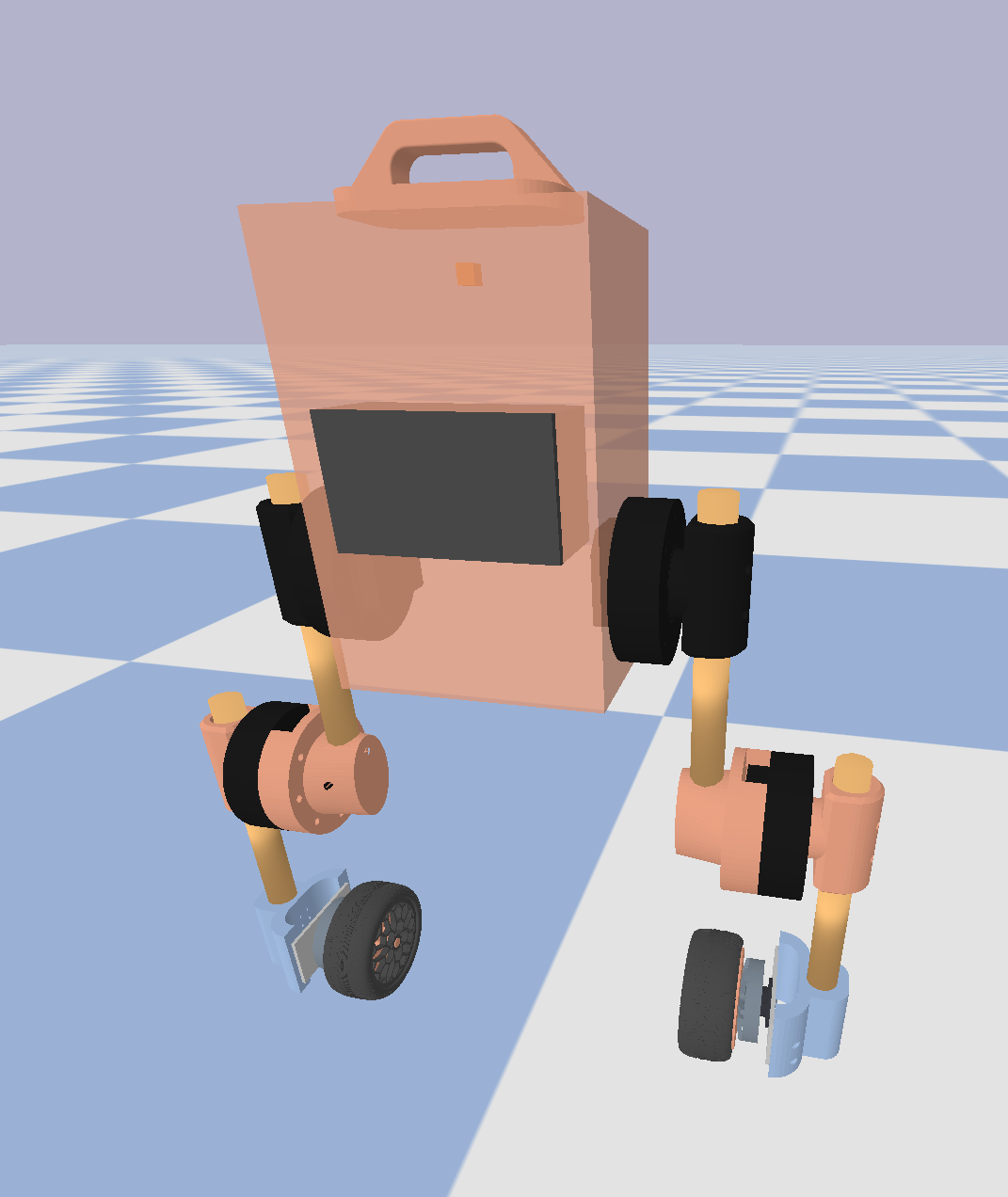}}
        \caption{\\ Upkie.}
        \vspace{-.3cm}
        \label{fig:upkie}
    }%
\end{wrapfigure}
\paragraph{Results} The 100-shot adaptation error of the control values is reported in~Table~\ref{table:adaptation}. The trajectories obtained with inverse dynamics control adapted from 50 shots are plotted in~Figure~\ref{fig:control} for~CAMEL and for the best-performing baseline,~ANIL, along with the analytical solution. Only~CAMEL adapts well enough to track the target trajectory. The analytic solution underestimates the control as it does not account for friction, resulting in inaccurate tracking. In the adaptive control setting, the variation in the mass of the cart leads to a deviation from the target trajectory but~CAMEL is able to adapt quickly to the new environment and identifies the new mass, unlike ANIL. Experimentation on~Upkie shows that the computational time of adaptation can be crucial, as we found that the gradient-based adaptation of~ANIL and~CoDA was too slow to run in the 200Hz model predictive control loop. On the other hand, CAMEL's gradient-free adaptation and interpretability allow it to track and identify changes in system dynamics, and to correctly predict the stabilizing control law.

\begin{figure}
    \begin{minipage}{.26\linewidth}
        \vspace{.3cm}
        \begin{subfigure}{\textwidth}
            \centering
            \includegraphics[width=\linewidth, trim={.5cm .5cm .0cm .5cm}
                ,clip]{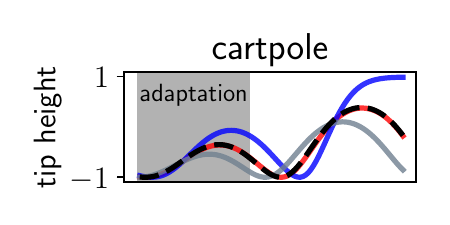}
            % \caption{Tracking error after adaptation.}
        \end{subfigure}
        % \hfill
        % \begin{subfigure}{0.15\textwidth}
        %     \centering
        %     \includegraphics[height=2cm]{legend_vertical.pdf}
        %     % \caption{Tracking error after adaptation.}
        % \end{subfigure}
        \par
        \begin{subfigure}{\textwidth}
            \centering
            \includegraphics[width=\linewidth, trim={.5cm .2cm .0cm .5cm}
                ,clip]{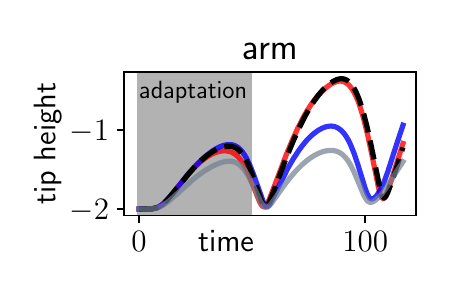}
            % \caption{Tracking error after adaptation.}
        \end{subfigure}
    \end{minipage}
    \begin{minipage}{.275\linewidth}
        \centering
        \includegraphics[width=\linewidth]{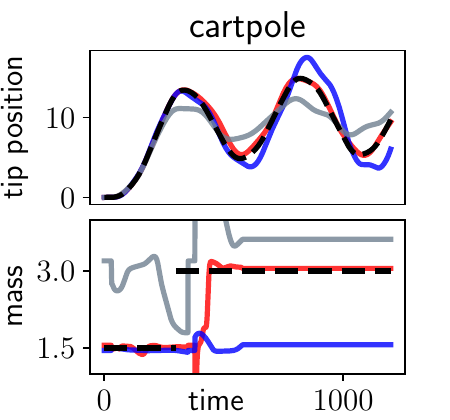}
    \end{minipage}
    \begin{minipage}{.35\linewidth}
        \centering
        \includegraphics[width=\linewidth]{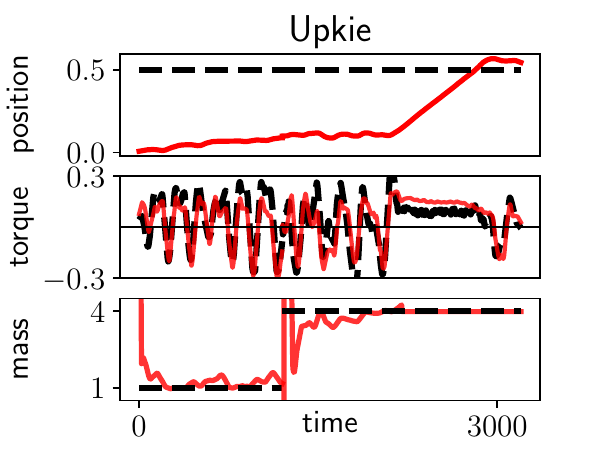}
    \end{minipage}
    \begin{minipage}{.06\linewidth}
        \centering
        \includegraphics[width=\linewidth]{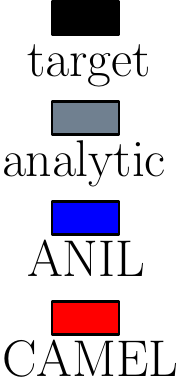}
    \end{minipage}
    \caption{Tracking of a reference trajectory using the learned inverse dynamics controller. \\ \textbf{Left.} 50-shot adaptation. \textbf{Center and right.} The model and the controller are adapted online.}
    \label{fig:control}
\end{figure}
\todo[inline]{parameter estimation with the analytic baseline}
% \begin{table}
%     \centering
%     \caption{Performances}
%     \begin{tabular}{l|llll}
%         {System}                   & MAML & ANIL & CoDA & CAMEL
%         % \\
%         % \hline \hline
%         % \multirow{2}{*}{dipole} & error  & -   &  -  & -   
%         % \\
%         %                           & time   & -   & -   & -    
%         \\
%         \hline
%         dipole      & -    & -    & - & -
%         \\
%         \hline
%         quadrupole      & -    & -    & - & -
%         \\
%         \hline
%         capacitor      & -    & -    & - & -
%         \\
%         \hline
%         cartpole      & -    & -    & - & -
%         \\
%         \hline
%         arm      & -    & -    & - & -
%         \\
%         \hlinedee
%     \end{tabular}
%     \label{table:time}
% \end{table}

% ------------------------------------------------------------
\subsection{Beyond context-linear systems}
\label{section:capacitor}

In order to evaluate our method on general systems with no known parametric structure, we consider the following non-analytical electrostatic problem of the form shown in~Example~\ref{example:potential}. The field is created by a capacitor formed by two electrodes that are not exactly parallel. The variability of the different experiments stems from the misalignment~$\delta \varphi\in \R^2$, in angle and position, of the upper electrode. We apply the same methodology as described in~Section~\ref{section:charges}. The whole multi-environment learning experiment is repeated several times with varying magnitudes of misalignment, by replacing~$\delta \varphi$ with~$\varepsilon \,\delta \varphi$ for different values of~$\varepsilon \in [0,1]$. This parameterization allows us to move gradually from local perturbations when~${\varepsilon \ll 1}$ (as in~Example~\ref{example:local-boundary}) to arbitrary variations in the environment.

% We call~$\varepsilon$-capacitor the corresponding systems.
% \newpage
\begin{wrapfigure}{r}{.48\linewidth}
    \centering
    % \vspace{-.4cm}
    \includegraphics[width=\linewidth, trim={.3cm .3cm .9cm 1.05cm}
        ,clip]{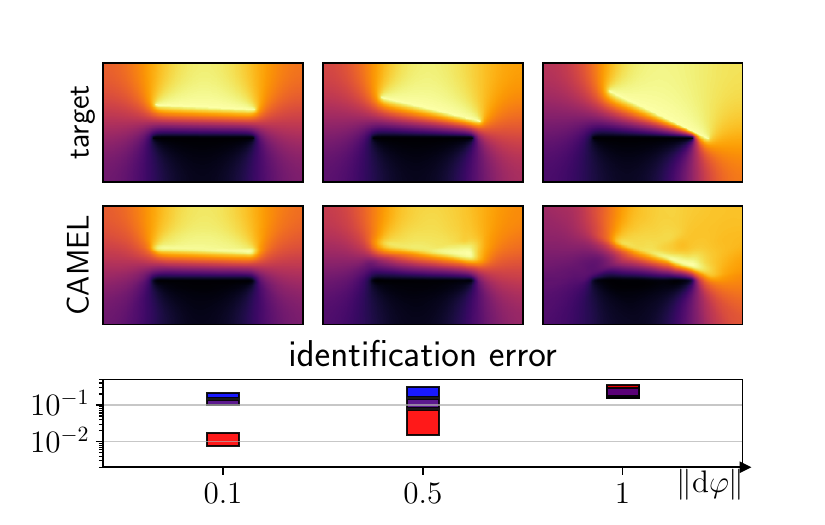}
    \par
    \includegraphics[width=.55\linewidth, trim={.0cm .03cm .05cm .01cm}
        ,clip]{legend_horizontal.pdf}
    \caption{Adaptation and relative identification error for the~$\varepsilon$-capacitor, with increasing~$\varepsilon$.}
    \label{fig:perturbation}
    \vspace{-.4cm}
\end{wrapfigure}
\paragraph{Results} The 40-shot adaptation error for the $\varepsilon$-capacitor is reported in~Table~\ref{table:adaptation}, with perturbation of full magnitude~$\varepsilon=1$ and with~$\varepsilon=0.1$. We also show the 5-shot adaptation of~CAMEL and the best performing baseline,~CoDA, for~$\varepsilon = 0.2$ in~Figure~\ref{fig:field}. When the system parameters are fully nonlinear,~CAMEL and the baselines perform similarly,~but CAMEL is much faster. In the second case,~CAMEL outperforms them by an order of magnitude and accurately predicts the electrostatic field, whereas~CoDA's exhibits lower precision. Predictions and average identification error (with standard deviations) are plotted as a function of~$\varepsilon$ in~Figure~\ref{fig:perturbation}. For small~$\varepsilon$, the system parameter perturbation is well identified, enabling a zero-shot adaptation. Remarkably,~Figure~\ref{fig:field} suggests that the zero-shot model~{$\varphi$-CAMEL} performs as well as its few-shot counterpart in this regime, demonstrating the effectiveness of interpretability.

% \begin{figure}
%     \centering
%     \includegraphics[width=\textwidth]{perturbation.pdf}
%     \caption{Adaptation as a function of the perturbation scale.}
%     \label{fig:perturbation}
% \end{figure}

% \todo[inline]{vector field?}

\captionof{table}{
    Average adaptation mean squared error {(\textbf{left})} and computational time {(\textbf{right})}.
}
% (arbitrary unit).
% }
% }
\begin{minipage}{.75\linewidth}
    \centering
    \resizebox{\textwidth}{!}{
        \begin{tabular}{c|c|c|c|c|c|c|}
            % \multirow{2}{*}
            {System}  & {Charges}               &
            % \multicolumn{2}{c|}{capacitor\, ; \, $\varepsilon=0.1$}
            Capacitor & $\varepsilon$-Capacitor
                      & {Cartpole}              & {Arm}           & Upkie
            % \\
            % \hline \hline
            % \multirow{2}{*}{dipole} & error  & -   &  -  & -   
            % \\
            %                           & time   & -   & -   & -    
            % \\
            % \hline
            % & {\footnotesize ($\times 10^{-3}$)} & {\footnotesize ($\times 10^{-2}$)} & ($\times 10^{-2}$) &       &
            \\
            \hline
            \hline
            MAML      & \textsc{1.6e-1}         & N/A             & N/A             & \textsc{1.8e0}  & \textsc{8.1e-1} & \textsc{1.5e-2}
            \\
            \hline
            ANIL      & \textsc{9.2e-4}         & \textsc{3.6e-2} & \textsc{1.1e-3} & \textsc{2.5e-2} & \textsc{7.5e-1} & \textsc{1.9e-2}
            \\
            \hline
            CoDA      & \textsc{8.2e-2}         & \textsc{2.6e-2} & \textsc{1.0e-3} & \textsc{8.1e-1} & \textsc{9.3e-1} & \textsc{2.1e-2}
            \\
            \hline
            {
            R2-D2 }     & \textsc{1.2e-4}         & \textsc{3.1e-4} & \textsc{4.2e-4} & \textsc{8.5e-3} & \textsc{3.5e-1} & \textsc{2.3e-2}
            \\
            \hline
            CAMEL     & \textsc{1.0e-4}         & \textsc{2.6e-2} & \textsc{1.9e-4} & \textsc{3.1e-3} & \textsc{2.4e-1} & \textsc{8.2e-3}
            % \\
            % \hline
            % $\varphi$-CAMEL     & \textsc{3.0e-3}         & \textsc{2.6e-2} & \textsc{6.6e-2} & \textsc{3.1e-3} & \textsc{2.4e-1} & \textsc{8.2e-3}
            % \\
            % \hline
        \end{tabular}
    }
    \label{table:adaptation}
\end{minipage}
% \qquad  \qquad
\begin{minipage}{.24\linewidth}
    % \vspace{.74cm}
    % \vspace{.44cm}
    \centering
    % \captionof{table}{Compute.}
    \resizebox{\textwidth}{!}{
        \begin{tabular}{|c|c|}
            {Training} & {Adaptation}
            % \\
            % \hline \hline
            % \multirow{2}{*}{dipole} & error  & -   &  -  & -   
            % \\
            %                           & time   & -   & -   & -    
            \\
            \hline
            \hline
            30         & 10
            \\
            \hline
            10         & 3
            \\
            \hline
            2          & 8
            \\
            \hline
            20          & 1
            \\
            \hline
            1          & 1
            \\
            % \hline
        \end{tabular}
    }
    \label{table:compute}
\end{minipage}

\newpage
% ====================================
\section{Related work}
% {\color{red}\st{Gradient-based}}
\paragraph{{{Multi-task}} meta-learning}
Meta-learning algorithms for multi-task generalization have gained popularity~\citep{hospedales2021meta}, with the MAML algorithm of~\cite{finn2017model} playing a fundamental role in this area. Based on the same principle, the variants~ANIL~\citep{raghu2019rapid} and~CAVIA~\citep{zintgraf2019fast} have been proposed to mitigate training costs and reduce overfitting. Interpretability is addressed in the latter work, using a large number of training tasks.
    {
        
        In a different line of work,~\cite{bertinetto2019meta} proposed the~R2-D2 architecture where the heads of the network are adapted using the closed-form formula of~Ridge regression.
        The similarities between multi-task representation learning and gradient-based learning are studied in~\citep{wang2021bridging} from a theoretical point of view, in the limit of a large number of tasks.
        Unlike our method, the approaches above rely on the assumption that the number of training tasks is large (in few-shot image classification for example, where it can be in the millions~\citep{wang2021bridging, hospedales2021meta}) and the number of data points per task is limited. For physical systems, in contrast, since experimenting is often costly, the number of tasks available at training is typically very limited, but the number of points for each task can be large. The assumption of limited allows the task-specific weightsto be stored in the meta-parameter vector instead of being computed at each training step. 
    }
% \citep{park2023firstorder}

\paragraph{Meta-learning physical systems}
Meta-learning has been applied to multi-environment data for physical systems, with a focus on  dynamical systems, where the target function is the flow of a differential equation. Recent algorithms include~LEADS~\citep{yin2021leads}, in which the task dependence is additive in the output space and CoDA~\citep{pmlr-v162-kirchmeyer22a}, where parameter identification is addressed briefly, but under strong assumptions of input linearity.
\cite{wang2022meta} propose physical-context-based learning, but context supervision is required for training. From a broader point of view, the interpretability of the statistical model can be imposed by adding physical constraints to the loss function~\citep{raissi2019physics}.

 {
 \paragraph{Multi-task reinforcement learning} 
 Meta-learning has given rise to a number of fruitful new approaches in the field of reinforcement learning.
%  In~\citep{clavera2018learning}, the forward dynamics is learned from multi-task data, as opposed to the function learned in our approach.
  \cite{sodhani2021multi} and~\cite{clavera2018learning} propose multi-task deep learning algorithms, but no structure is assumed on the dynamics and the learned weights can be interpreted only statistically, in the parameter space of a large black-box neural network. 
%   System identification seems allowing for , the control policy is directly meta-learned, in a model-free fashion. This method is arguably more flexible and the weights of the networks are interpreted statistically. Although the meta-learned models in these approaches can then be directly exploited using model-free or model-based control algorithms, no a priori structure of the learned function is known, and interpreting the task parameters and performing adaptive system identification therefore seems out of reach.
   Multi-task learning of inverse dynamics with varying inertial parameters is studied in~\citep{williams2008} using~Gaussian processes, but parameter identification is not addressed.
 }

% \paragraph{Multi-task reinforcement learning} \citep{sodhani2021multi}

% \paragraph{Interpretability} \cite{pmlr-v162-kirchmeyer22a,sodhani2021multi,zintgraf2019fast}

% Meta-learning Hamiltonian~\cite{lee2021identifying}

% Our work fundamentally differs from MAML and ANIL in the structural assumpmtion, and the way it is trained. In particular, our meta-learning algorithm does not have an inner loop, resulting in a first-order gradient descent.

% --------------------------------
\section{Conclusion}
\todo[inline]{Interpretable for linear to approximately linear contexts }
\todo[inline]{NTK? }
\todo[inline]{Model-agnostic so generalizes to convolutional neural networks}

We introduced~CAMEL,
% {\color{red}\st{a new approach to}}
{a simple multi-task learning algorithm designed for} multi-environment learning of physical systems. For general and complex physical systems, we demonstrated that our method performs as well as the state-of-the-art, at a much lower computational cost. Moreover, when the learned system exhibits a linear structure in its physical parameters, our architecture is particularly effective, and enables the identification of these parameters with little supervision, independently of training.
The identifiability conditions found in~Proposition~\ref{proposition:identification} are not very restrictive, and the effectiveness of the linear identification map is demonstrated in our experiments.

We proposed a particular application in the field of robotics where our data-driven method enables concurrent adaptive control and system identification. We believe that enforcing more physical structure in the meta-model, using for example~Lagrangian neural networks~\citep{lutter2019deep}, can improve its sample efficiency and extend its applicability to more complex robots.

While we focused on classical regression tasks, our framework can be generalized to predict dynamical systems by combining it with a differentiable solver~\citep{chen2018neural}.
Another interesting avenue for future research is the use of active learning, to make the most at out the available training resource and enhance the efficiency of multi-task learning for static and dynamic systems~\citep{pmlr-v202-wang23b, blanke2023flex}.
% % Multi-task linear regression
% % When $v(x; \Theta) = \Theta \times x$,
% % % \begin{equation}
% % $y = \transp{w_{t}} \Theta \times x$.

% Invertible networks \cite{ardizzone2018analyzing}

% Combine with a differentiable solver for prediction of dynamical systems.

% With a large number of contexts, context-supervised learning for interpretability.

% bridge the gap between the power of neural networks and the transparency required when dealing with the complexities of the physical world.

\clearpage
% =================================================

\section*{Acknowledgements}
This work was partially supported by the French government
under management of Agence Nationale de la Recherche as
part of the “Investissements d’avenir” program, reference
~ANR19-P3IA-0001 (PRAIRIE 3IA Institute).

\bibliography{references}
\bibliographystyle{iclr2024_conference}
\clearpage
\appendix
\section{Proofs}
\label{appendix:proof}

\begin{lemma}
    \label{lemma:symmetry}
    Let~$v_1,\dots, v_N$, and~$w_1,\dots, w_T \in\R^r$, and let~$r'\leq r$ and~$v'_1, \dots, v'_N$, and~${w'_1,\dots, w'_T \in \R^{r'}}$ be two sets of vector of full rank, satisfying~$\forall i, t, \transp{w_t}v_i = \transp{w'}_t v'_i$. Then there exist~${P, Q\in\R^{r'\times r}}$ such that~$w_t' = {P}  w_t$ and~$v'_i = Q v_i$. Furthermore,~${Q}\transp{P} = I_{r'}$.
    % if~$r=r'$, then~$P, Q{\in \mathrm{GL}_r(\R)}$ and~$\transp{Q}=P^{-1}$.
\end{lemma}

\begin{proof}[Proof of Lemma~\ref{lemma:symmetry}]
    Denoting by~$V\in \R^{N \times r}$,~$V'\in \R^{N \times r'}$,~$W \in \R^{T \times r}$ and~$W' \in \R^{T \times r'}$ the matrix representations of the vectors, the scalar equalities~$\forall i, t, \transp{w_t}v_i = \transp{w'}_t v'_i$ take the matrix form
    \begin{equation}
        \label{eq:matrix_equality}
        V \transp{W} = V' \transp{W'}.
    \end{equation}
    Since~$V'$ is of full rank, the matrix~$V'^+ := (V' \transp{V'})^{-1}\transp{V'} \in \R^{r \times N}$ is well defined and is a left inverse of~$V'$. Multiplying~\eqref{eq:matrix_equality} by~$V'^+$ yields
    \begin{equation}
        W' = W \transp{P} \quad \text{with} \quad {P} := {{V'^+}V} \; \in \R^{r' \times r}.
    \end{equation}
    Similarly,
    \begin{equation}
        V' = V \transp{Q} \quad \text{with} \quad {Q} := {{W'^+}W} \; \in \R^{r' \times r}.
    \end{equation}

    Now compute~${Q}\transp{P} = {W'^+}W \transp{P} = {W'^+}W' = I_{r'}$

\end{proof}

\begin{proof}[Proof of Proposition~\ref{proposition:identification}]
    Applying Lemma~\ref{lemma:symmetry} to~$v'_i := \nu(x^{(i)})$,~$v_i := v(x^{(i)})$, and~$w_t:= \omega_t$,~$w'_t:=\varphi_t$ yields the stated result.
\end{proof}
{
The case  where~$c, \kappa \neq 0$ can be handled as follows. We augment~$\varphi$ and $\nu$, and $\omega$ and $v$ with an additional dimension, with the last components of~$\varphi$ and $\omega$ equal to $1$ and the last components of~$\nu$ and~$v$ equal to~$\kappa$ and $c$ respectively. The augmented vectors satisfy the assumptions of~Proposition~\ref{proposition:identification} provided the augmented~$v'_i$ and~$w'_t$ span~$\R^{n+1}$. The proposition then applies, and implies that the physical parameters~$\varphi_t$ can be recovered with an affine transform. This case is tackled experimentally in the capacitor experiment~(Section~\ref{section:capacitor}), where~$\kappa\neq 0$ \textit{a fortiori} since the electrostatic field is linearized around a nonzero value. The physical parameters are identified using an affine regression.
}

% =================================================
\section{Experimental details}
\label{appendix:experiments}

{

% --------------------------------------------
\subsection{Architectures}

All neural networks are trained with the ADAM optimizer~\cite{2015-kingma}. For CoDA, we set~$d_\xi =r$, chosen according to the system learned. For all the baselines, the adaptation minimization problem~\eqref{eq:adaptation_problem} is optimized with at least 10 gradient steps, until convergence.

For training, the number of inner gradient steps of MAML and ANIL is chosen to be 1, to reduce the computational time. We have also experimented with larger numbers of inner gradient steps. This improved the stability of training, but at the cost of greater training time.}

% --------------------------------------------
\subsection{Systems}

We provide further details about the physical systems on which the experiments of~Section~\ref{section:experiments} are performed.

% ......................................
\subsubsection{Point charges}

The~$n$ charges are placed at fixed locations in the plane at fixed location.
The training inputs are located in~$\Omega=[-1, 1]\times[0, 1]$ which is discretized into a~$20\times20$ grid and the ground truth potential field is computed using~Coulomb's law.

The training data is generated by changing each charge's value in~$\{1,\dots, 5\}^n$, hence~$T=5^n$.
We have experimented on different settings with various numbers of charges, and various locations. In~Section~\ref{section:charges}, a dipolar configuration is investigated, where~$n=3$, and one of the charges is far away on the left and two other charges of opposite sign are located near~$x_2=0$. Gaussian noise of size~$\sigma=0.1$ is added to the field values revealed to the learner in the test dataset.

The system is learned with a neural network of 4 hidden layers of width 16, with the last layer of size~$r=n$.

For evaluation, the test data is generated with random charges drawn from a uniform distribution in~$[1,\dots, 5]^n$ and the data points are drawn uniformly in~$\Omega$

% ......................................
\subsubsection{Capacitor}

The space is discretized into a $200\times300$ grid. The training environments are generated with 10 values of the physical context~$\varphi:=(\alpha, \eta)\in [0, 0.5]\times[-0.5, 0.5]$ containing the angular and the positional perturbation of the second plate, drawn uniformly.
The ground truth electrostatic field is computed with the Poisson equation solver of~\cite{electronics11152365}.
For evaluation, 5 new environments are drawn with the same distribution.

The system is learned with a neural network of 4 hidden layers of width 64, with the last layer of size~$r=n+1 = 3$.

% ......................................
\subsubsection{Cartpole and arm}

We have implemented the manipulator equations for the cartpole and the arm (or acrobot), following~\cite{underactuated}, and have added friction. The training data is generated by actuating the robots with sinusoidal inputs, with for each environment 8 trajectories of 200 points and random initial conditions and periods. At test time, the trajectories are generated with sinusoidal inputs for evalutation, and with swing-up inputs for trajectory tracking.

\paragraph{Cartpole} The pole's length is set to 1, the varying physical parameters are the masses of the cart and of the pole: $\varphi_t \in \{1, 2\}\times \{0.2, 0.5\}$, so~$T=4$. For evalutation, the masses are drawn uniformly around~$(2, 0.3)$, with an amplitude of~$(1, 0.2)$.
The system is learned with a neural network of 3 hidden layers of width 16, with the last layer of size~$r=n+2 = 4$.

\paragraph{Arm} The arm's length are set to 1, the varying physical parameters are the inertia and the mass of the second arm: $\varphi_t \in \{0.25, 0.3, 0.4\}\times \{0.9, 1.0, 1.3\}$, so $T=9$. For evalutation, the inertial parameters are drawn uniformly around~$(0.5, 1)$, with an amplitude of~$(0.2, 0.3)$.
The system is learned with a neural network of 4 hidden layers of width 64, with the last layer of size~$r=n+2 = 4$.

% ......................................
\subsubsection{Upkie}

Information about the open-source robot Upkie can be found at~\url{https://github.com/tasts-robots/upkie}.

We trained the meta-learning algorithm on balancing trajectories of 1000 observations, with 10 different values for Upkie's torso, ranging from 0.5 to 10 kilograms. For evaluation, the mass is sampled in the same interval.

The system is learned with a neural network of 4 hidden layers of width 64, with the last layer of size~$r=n+2 = 3$.

% --------------------------------------------
\subsection{Inverse dynamics control}
\label{appendix:inverse-dynamics-control}

Inverse dynamics control is a nonlinear control technique that aims at computing the control inputs of a system given a target trajectory~$\{ \bar{q}(s)\}$~\cite{spong2020robot}. Using a model $\hat{\mathrm{ID}}$ for the inverse dynamics equation~\eqref{example:inverse-dynamics}, the feedforward predicted control signal~$\hat{u} = \hat{\mathrm{ID}}(\bar{q}, \dot{\bar{q}}, \ddot{\bar{q}})$. These feedforward control values can then be combined with a low gain feedback controller to ensure stability, as
\begin{equation}
    u = \hat{u} + K(\bar{q} - q) + K'(\dot{\bar{q}} - \dot{q}).
\end{equation}

For the cartpole, we used~$K= K' =0.5$. For the robot arm, we used~$K=K'=1$.

% --------------------------------------------
\subsection{Adaptive control}
\label{appendix:control}

In a time-varying dynamics scenario, CAMEL can be used for adaptive control and system identification. Given a target trajectory, the task-agnostic component~$v$ of the model predictions can be computed offline. In the control loop, the task-specific component~$\omega$ is updated with the online least squares formula. The control loop is summarized in~Algorithm~\ref{algorithm:online_control}, where we have assumed~$c=0$ for simplicity. The estimated inertial parameters are deduced from the task-specific weights with the identification matrix~\eqref{eq:linear_identification}.

\begin{center}
    \begin{algorithm}[H]
        \caption{Adaptive trajectory tracking}
        \label{algorithm:online_control}

        % \begin{multicols}{2}
        \begin{algorithmic}
            \State \textbf{input}
            trained feature map~$v(x)$, target trajectory~${s \mapsto \bar{q}_s}$
            \State \textbf{Offline control}
            \For{timestep $0 \leq s \leq H-1$}
            \State  compute $\bar{x}_s  = (\bar{q}_s, \dot{\bar{q}}_s, \ddot{\bar{q}}_s) $
            \State compute features $\bar{v}_s := v(\bar{x}_s)$
            \EndFor
            % \vfill\null
            % \columnbreak
            \State \textbf{Control loop}
            \State Initialize $M_0 = I_r$, \quad $\omega_0 = (0, \dots, 0)$
            \For{time step $1 \leq s \leq H$}
            % \State  instantiate $f_s(x; \pi) = f(x ; \theta, \omega_s)$
            \State  compute $\hat{u}_s = \transp{\omega}_s \bar{v}_s$
            \State  compute $e_s = q_s - \bar{q}_s$
            \State  play $u_s := \hat{u}_s + K e_s$
            \State  observe~$q_{s+1}$, $\dot{q}_{s+1}$
            \State  compute~$v_s := v(x_s)$
            \State  update $M_{s+1} = M_{s} -\frac{M_{s}v_s\transp{(M_s v_s)}} {1+ \transp{v}_s M_s v_s}$
            \State  update $\omega_{s+1} = \omega_s - (\transp{v_{s}}\omega_s - u_s)M_{s+1}v_{s}$
            % \State  update $\omega_{s+1} = w_{s} -(\transp{v}_{s+1}w_s - u_s) M_{s+1} v_{s}$
            \EndFor
        \end{algorithmic}
        % \end{multicols}
    \end{algorithm}
\end{center}

\clearpage

% --------------------------------------------
\subsection{Additional numerical results}
We provide details concerning~Table~\ref{table:adaptation}.

{
\paragraph{Computational time} For the computational times of~Table~\ref{table:adaptation}, we arbitrarily chose the shortest time as the time unit, for a clearer comparison among the baselines. The computational times were measured and averaged over each experiment, with equal numbers of batch sizes and gradient steps across the different architectures. For training, the time was divided by the number of gradient steps.
}

% \begin{tabular}{c|c|c|c|c|c|}
%     \multirow{2}{*}{System} & {dipole}           & {capacitor}        & {cartpole}         & {arm} & upkie
%     % \\
%     % \hline \hline
%     % \multirow{2}{*}{dipole} & error  & -   &  -  & -   
%     % \\
%     %                           & time   & -   & -   & -    
%     \\
%     % \hline
%                             & ($\times 10^{-3}$) & ($\times 10^{-2}$) & ($\times 10^{-2}$) &       &
%     \\
%     \hline
%     \hline
%     MAML                    & $3500 \pm 2000$    & N/A                & $180 \pm 20$       & -     & -
%     \\
%     \hline
%     ANIL                    & $3.0 \pm 2$        & $2.7 \pm 1.1$      & $2.5 \pm 1.8$      & -     & -
%     \\
%     \hline
%     CoDA                    & $2500 \pm 150$     & $2.1 \pm  1.3$     & $80 \pm 30$        & -     & -
%     \\
%     \hline
%     CAMEL                   & $0.14 \pm 0.2$     & $2.1 \pm 1.0 $     & $0.3 \pm 0.07$     & -     & -
%     \\
%     % \hline
% \end{tabular}

\begin{table}[H]
    \centering
    \caption{Adaptation performances with standard deviations.}
    \begin{tabular}{c|c|c|c|c|}
        \multirow{2}{*}{System} & \multicolumn{2}{c|}{Charges{,  30 trials}} & \multicolumn{2}{c|}{Capacitor{,  5 trials}}
        % \\
        % \hline \hline
        % \multirow{2}{*}{dipole} & error  & -   &  -  & -   
        % \\
        %                           & time   & -   & -   & -    
        \\
        % \hline
                                & 3-shot                                                 & 10-shot                                                 & 5-shot                     & 40-shot
        \\
        \hline
        \hline
        MAML                    & \textsc{4.1e-0 $\pm$ 2e-0}                             & \textsc{1.6e-1 $\pm$ 5e-2}                              & N/A                        & N/A
        \\
        \hline
        ANIL                    & \textsc{3.5e0 $\pm$ 5e-1}                              & \textsc{9.2e-4 $\pm$ 5e-4}                              & \textsc{4.4e-2 $\pm$ 2e-2} & \textsc{3.6e-2$\pm$ 1e-2}
        \\
        \hline
        CoDA                    & \textsc{1.0e-1 $\pm$ 9e-2}                             & \textsc{8.2e-2 $\pm$ 3e-2}                              & \textsc{4.7e-2 $\pm$ 5e-5} & \textsc{2.6e-2$\pm$ 1e-2}
        \\
        \hline
        CAMEL                   & \textsc{2.0e-4 $\pm$ 1e-4}                             & \textsc{1.0e-4 $\pm$ 5e-5}                              & \textsc{3.6e-2 $\pm$ 2e-2} & \textsc{2.6e-2 $\pm$ 1e-2}
        \\
        \hline
        \hline
        $\varphi$-CAMEL         & \multicolumn{2}{c|}{\textsc{3.0e-3} }                  & \multicolumn{2}{c|}{\textsc{6.5e-2}}
        \\
        \hline
    \end{tabular}
    \label{table:time}
\end{table}
\begin{table}[H]
    \centering
    % \caption{Adaptation performances with standard deviations.}
    \begin{tabular}{c|c|c|}
        \multirow{2}{*}{System} & \multicolumn{2}{c|}{$\varepsilon$-Capacitor, $\varepsilon=0.1${, 5 trials}}
        % \\
        % \hline \hline
        % \multirow{2}{*}{dipole} & error  & -   &  -  & -   
        % \\
        %                           & time   & -   & -   & -    
        \\
        % \hline
                                & 3-shot                                                                                  & 30-shot
        \\
        \hline
        \hline
        MAML                    & N/A                                                                                     & N/A
        \\
        \hline
        ANIL                    & \textsc{1.1e-3 $\pm$ 5e-5}                                                              & \textsc{1.1e-3 $\pm$ 5e-5}
        \\
        \hline
        CoDA                    & \textsc{1.2e-3 $\pm$ 5e-4}                                                              & \textsc{1.0e-3 $\pm$ 5e-4}
        \\
        \hline
        CAMEL                   & \textsc{4.2e-4 $\pm$ 1e-4}                                                              & \textsc{1.9e-4 $\pm$ 2e-5}
        \\
        \hline
        \hline
        $\varphi$-CAMEL         & \multicolumn{2}{c|}{\textsc{1.9e-4}}
        \\
        \hline
    \end{tabular}
\end{table}
\begin{table}[H]
    \centering
    % \caption{Adaptation performances with standard deviations.}
    \begin{tabular}{c|c|c|c|c|}
        \multirow{2}{*}{System} & \multicolumn{2}{c|}{Cartpole{, 50 trials}} & \multicolumn{2}{c|}{Arm{, 50 trials}}
        % \\
        % \hline \hline
        % \multirow{2}{*}{dipole} & error  & -   &  -  & -   
        % \\
        %                           & time   & -   & -   & -    
        \\
        % \hline
                                & 50-shot                                                & 100-shot                                          & 50-shot                    & 100-shot
        \\
        \hline
        \hline
        MAML                    & \textsc{4.3e0 $\pm$ 7e-1}                              & \textsc{3.5e0 $\pm$ 6e-1}                         & \textsc{1.0e0 $\pm$ 1e-1}  & \textsc{8.1e-1 $\pm$ 5e-2}
        \\
        \hline
        ANIL                    & \textsc{3.8e-1 $\pm$ 1e-1}                             & \textsc{2.5e-2 $\pm$ 9e-2}                        & \textsc{8.5e-1 $\pm$ 1e-1} & \textsc{7.5e-1 $\pm$ 4e-2}
        \\
        \hline
        CoDA                    & \textsc{3.8e-1 $\pm$ 9e-3}                             & \textsc{8.1e-1 $\pm$ 1e-1}                        & \textsc{9.5e-1 $\pm$ 9e-2} & \textsc{9.3e-1 $\pm$ 6e-2}
        \\
        \hline
        CAMEL                   & \textsc{4.8e-2 $\pm$ 1e-2}                             & \textsc{3.1e-3 $\pm$ 5e-4}                        & \textsc{3.1e-1 $\pm$ 5e-2} & \textsc{2.4e-1 $\pm$ 1e-2}
        \\
        \hline
    \end{tabular}
\end{table}
\begin{table}[H]
    \centering
    % \caption{Adaptation performances with standard deviations.}
    \begin{tabular}{c|c|c|}
        % \multirow{2}{*}
        {System} &
        % \multicolumn{2}{c|}
        {Upkie{, 15 trials}}
        % \\
        % \hline \hline
        % \multirow{2}{*}{dipole} & error  & -   &  -  & -   
        % \\
        %                           & time   & -   & -   & -    
        \\
        \hline
        \hline
        MAML     & \textsc{1.5e-2 $\pm$ 7e-3}
        \\
        \hline
        ANIL     & \textsc{1.9e-2 $\pm$ 6e-3}
        \\
        \hline
        CoDA     & \textsc{2.1e-2 $\pm$ 3e-3}
        \\
        \hline
        CAMEL    & \textsc{8.2e-3 $\pm$ 5e-3}
        \\
        \hline
    \end{tabular}
\end{table}

\clearpage

\begin{figure}[H]
    \centering
    \includegraphics[width=.8\linewidth]{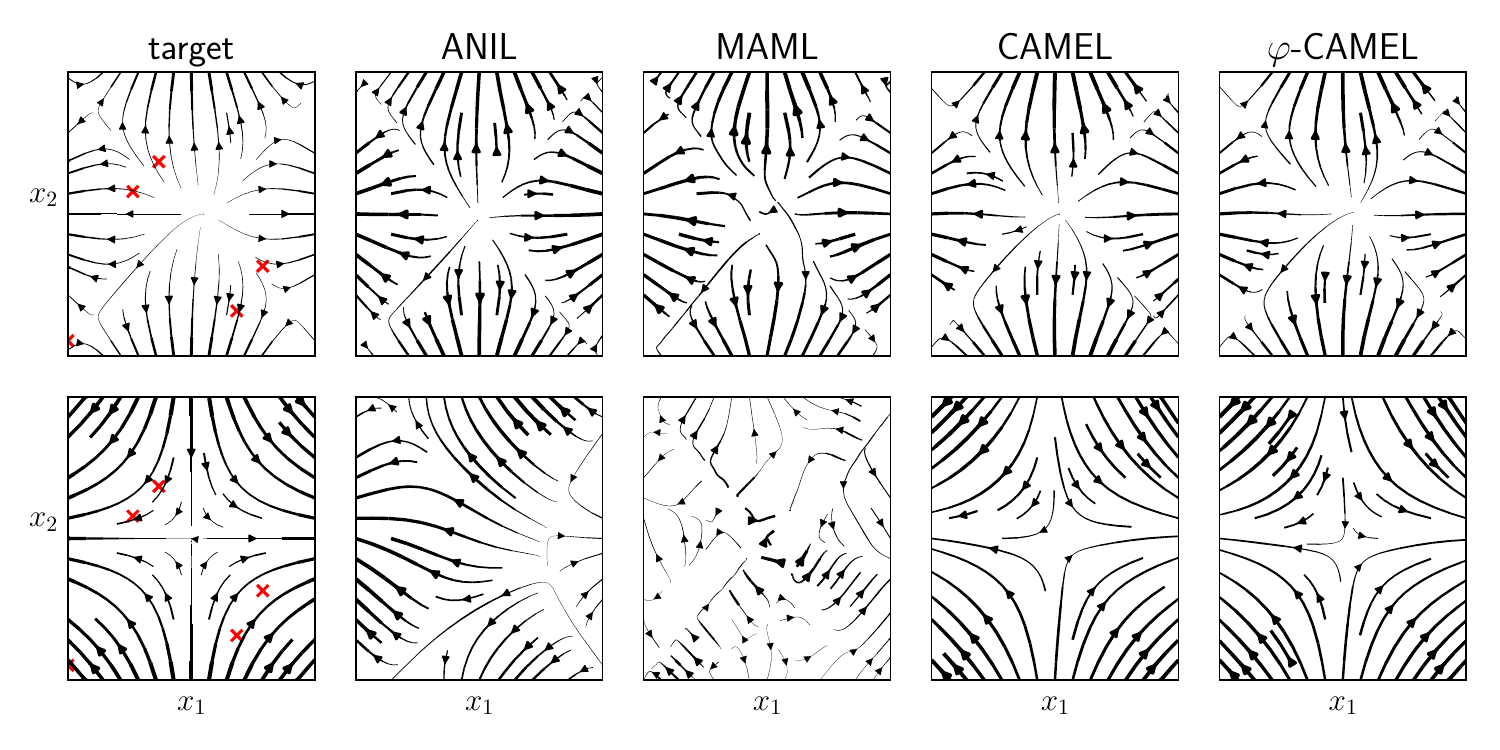}
    \caption{5-shot adaptation for the 4 point charge system. \textbf{Top.} The four charges are positive, as in the training meta-dataset. \textbf{Bottom} Two of the four charges are negative.}
    \label{fig:quadrupole}
\end{figure}
\begin{figure}[H]
    \centering
    \includegraphics[width=.8\linewidth]{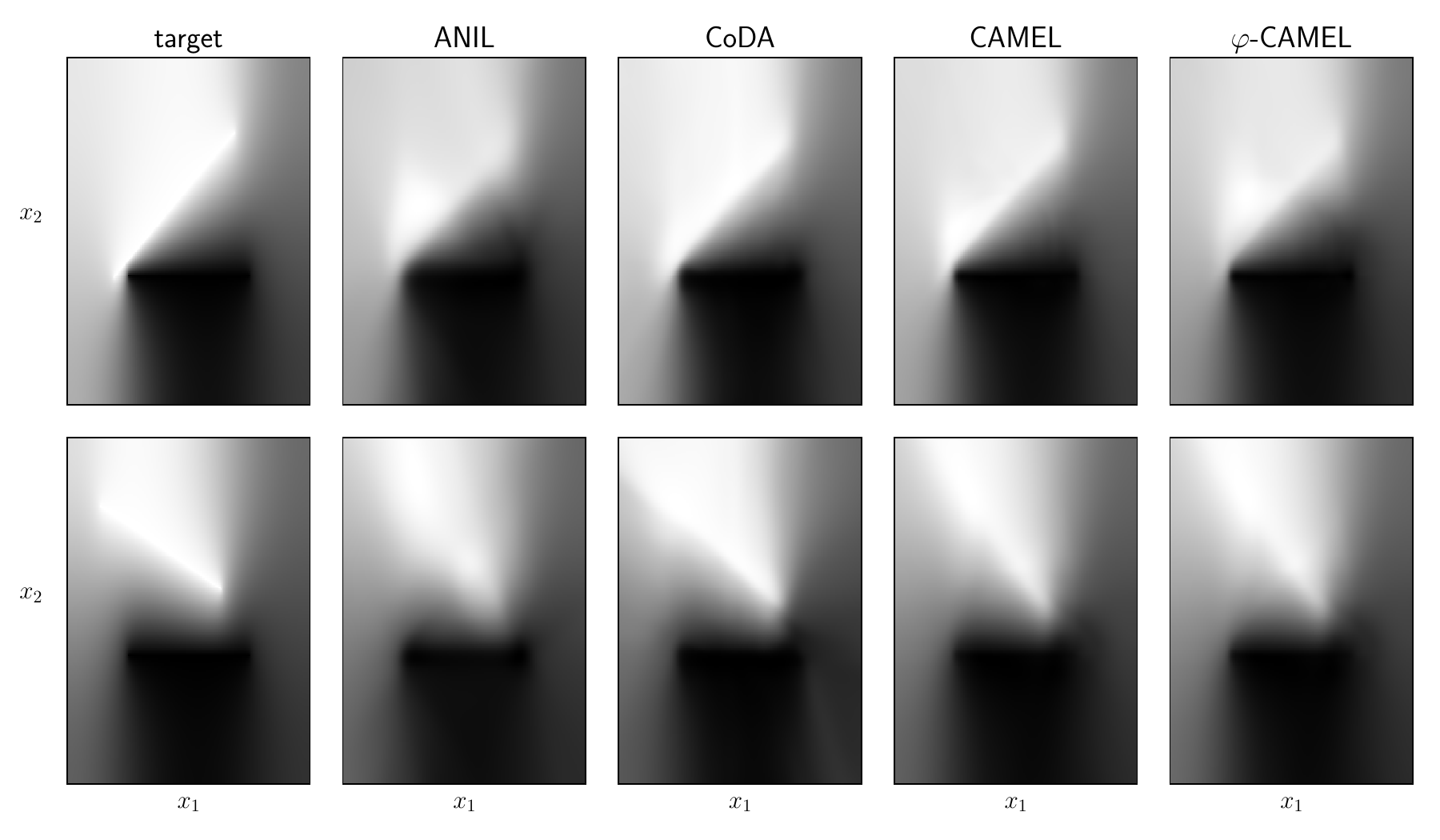}
    \caption{Capacitor, 40-shot adaptation. }
    \label{fig:capacitor_full}
\end{figure}
\begin{figure}[H]
    \centering
    \includegraphics[width=.5\linewidth]{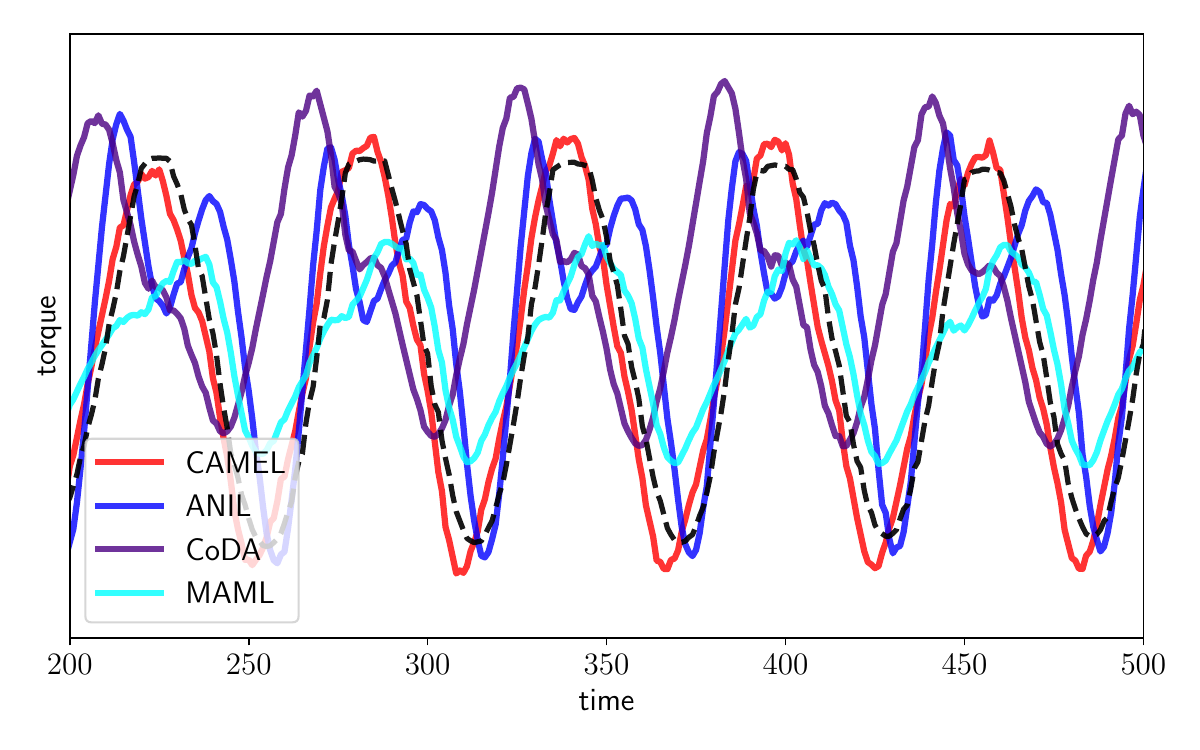}
    \caption{Upkie torque prediction, 100-shot adaptation. }
    \label{fig:benchmark_upkie}
\end{figure}

% --------------------------------------------
\subsection{Zero-shot adaptation and scientific discovery}
\label{appendix:discovery}

In a data-driven approach, training CAMEL offers not only the ability to adapt to a small number of observations, but also to predict the system without any data for arbitrary values of the its parameters.  We believe that the 0-shot adaptation algorithm~$\varphi$-CAMEL that we introduced in~Section~\ref{section:identification} can be used in the process of scientific discovery. In many cases, the experimenter has the knowledge of~(or knows an estimate of) the physical quantities varying across experimental conditions, while not knowing accurately the system itself. Then, $\varphi$-CAMEL can be used to infer the target function for chosen values of the physical parameters~$\varphi$ independently of the values observed for training.

Of course, the predictions of~$\varphi$-CAMEL are good only if the estimator~$\hat{\varphi}$ of~\eqref{eq:linear_identification} is good, implying a sufficient number of training tasks and an effective training of CAMEL. For nonlinear physical contexts, the values of~$\varphi$ that are investigated should be close to the reference value~$\varphi_0$ so that~\eqref{eq:locally-linear} holds.

We further illustrate on the toy example of~$n=4$ point charges, for which the experimenter could observe experiments with positive charges.~Figure~\ref{fig:quadrupole} shows the predictions after 5-shot adaptation of the different meta-models, along with the zero-shot adaptation of~$\varphi$-CAMEL. We can see that only CAMEL and $\varphi$-CAMEL adapt well to negative charges. In particular, the zero-shot adaptation of $\varphi$-CAMEL enables estimating the system in an experiment whose numerical values are completely different from the training dataset, thanks to the structure of the model and of the equations in this case (since they are known to be linear in the charges). Importantly, evaluating $\varphi$-CAMEL for different values of~$\varphi$ is not costly, since the identification map is already computed using the training data.

We could imagine that this scenario might enable discovering new properties of complex physical systems as by exploring the space of physical parameters, in a data-driven fashion. Regarding the simple example of~Figure~\ref{fig:quadrupole}, knowing the form of the electrostatic field in this quadrupole setting underlies the understanding of Penning's ion trap~\cite{kretzschmar1991particle}.

\end{document}